%% file: ms.tex
\pdfoutput=1

\documentclass[twoside]{article}

\usepackage[accepted]{aistats2025}
\usepackage[round]{natbib}

\usepackage{bm}
\usepackage{xcolor}
\usepackage{amsmath}
\usepackage{amssymb}
\usepackage{amsthm}
\usepackage{amsfonts}       % blackboard math symbols
\usepackage{algorithm}
\usepackage{algorithmic}
\usepackage{subcaption}
\usepackage{graphicx}
\usepackage{diagbox} % add diagonal box to real data table
\graphicspath{ {./img/} }

\usepackage{Sty/mcr}

\newtheorem{lemma}{Lemma}
\newtheorem{remark}{Remark}
\newtheorem{example}{Example}

\DeclareMathOperator*{\argmin}{arg\,min}
\DeclareMathOperator*{\argmax}{arg\,max}

\begin{document}

% If your paper is accepted and the title of your paper is very long,
% the style will print as headings an error message. Use the following
% command to supply a shorter title of your paper so that it can be
% used as headings.
%
\runningtitle{Controllable RANSAC-based Anomaly Detection via Hypothesis Testing}

% If your paper is accepted and the number of authors is large, the
% style will print as headings an error message. Use the following
% command to supply a shorter version of the authors names so that
% they can be used as headings (for example, use only the surnames)
%
\runningauthor{Le Hong Phong, Ho Ngoc Luat, Vo Nguyen Le Duy}

\twocolumn[

\aistatstitle{Controllable RANSAC-based Anomaly Detection via \\Hypothesis Testing}

\aistatsauthor{Le Hong Phong$^{1, 2}$, Ho Ngoc Luat$^{1, 2}$, Vo Nguyen Le Duy$^{1, 2, 3, \ast}$}
\vspace{5pt}
\aistatsaddress{ 
$^1$University of Information Technology, Ho Chi Minh City, Vietnam \\ 
$^2$Vietnam National University, Ho Chi Minh City, Vietnam \\
$^3$RIKEN
}]

\begin{abstract}
\input{abst}

\end{abstract}

\input{sec1}

\input{sec2}

\input{sec3}
\input{sec4}
\input{sec5}

\bibliographystyle{abbrvnat}
\bibliography{ref}

\newpage
\onecolumn

\input{appendix}

\end{document}

%% file: abst.tex
Detecting the presence of anomalies in regression models is a crucial task in machine learning, as anomalies can significantly impact the accuracy and reliability of predictions. 
Random Sample Consensus (RANSAC) is one of the most popular robust regression methods for addressing this challenge.
However, this method lacks the capability to guarantee the reliability of the anomaly detection (AD) results.
In this paper, we propose a novel statistical method for testing the AD results obtained by RANSAC, named CTRL-RANSAC (controllable RANSAC).
The key strength of the proposed method lies in its ability to control the probability of misidentifying anomalies below a pre-specified level $\alpha$ (e.g., $\alpha = 0.05$).
By examining the selection strategy of RANSAC and leveraging the Selective Inference (SI) framework, we prove that achieving controllable RANSAC is indeed feasible.
Furthermore, we introduce a more strategic and computationally efficient approach to enhance the true detection rate and overall performance of the CTRL-RANSAC.
Experiments conducted on synthetic and real-world datasets robustly support our theoretical results, showcasing the superior performance of the proposed method.

%% file: sec1.tex
\section{Introduction} \label{sec:intro}

\begin{figure*}[!t] 
\centering
    \includegraphics[width=\linewidth]{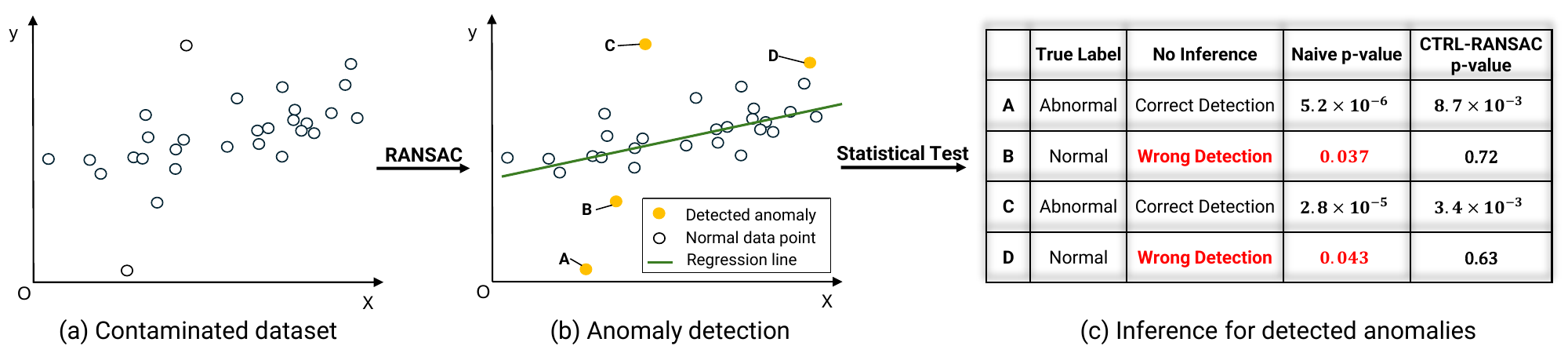}
    \caption{Illustration of the proposed CTRL-RANSAC method. Performing AD without inference produces wrong anomalies (\textbf{B, D}). The naive $p$-values are even small for falsely detected anomalies. With the  $p$-value provided by CTRL-RANSAC, we can identify both false positive (FP) and true positive (TP) detections, i.e., large \textit{p}-values for FPs and small \textit{p}-values for TPs.}
    \label{fig:illustration}
    \vspace{-5pt}
\end{figure*}

Anomaly detection (AD) in regression models is a crucial problem in machine learning, as highlighted by  \cite{aggarwal2017outlier}. 
Identifying these anomalies is crucial,
% for ensuring the accuracy and reliability of model predictions, 
as they can impact the performance and interpretability of regression analyses.
One effective approach for handling this task is the RANSAC \citep{fischler1981random}.
%which is designed to identify and mitigate the influence of outliers by iteratively fitting a model to subsets of the data and evaluating the consensus among these subsets. 
%
Compared to methods like least absolute deviation (LAD) regression \citep{andrews1974robust} and Huber regression \citep{huber1973robust}, RANSAC offers advantages in handling datasets with many outliers. While Huber and LAD regressions mitigate the influence of outliers by minimizing the Huber loss or absolute deviations, they still incorporate all data points, which can bias the results. In contrast, RANSAC iteratively selects subsets of inliers and discards outliers entirely, leading to more accurate model fits.
RANSAC has been widely applied in several applications such as wide-baseline matching \citep{mishkin2015mods}, multi-model fitting \citep{isack2012energy}, and structure-from-motion \citep{schonberger2016structure, barath2021efficient}.

There is a critical concern regarding the potential for erroneous results.
RANSAC has a high tendency to detect anomalies even when none exist in the data generation process. 
This can cause serious implications, especially in high-stakes decision-making scenarios.
For example, in the medical field, mistakenly identifying a healthy individual as unhealthy could lead to unnecessary and potentially harmful treatments. 
Similarly, in cybersecurity, incorrect identification of cyberattacks can waste resources and result in substantial financial losses.
These errors are commonly denoted as \emph{false positives}, and there is a critical necessity of properly controlling the false positive rate (FPR).

We also emphasize the importance of controlling the false negative rate (FNR). 
In the literature of statistics, a common approach is to first control the FPR at a pre-determined level $\alpha$, e.g., 0.05, while simultaneously striving to minimize the FNR, i.e., maximizing the true positive rate (TPR $=$ 1 $-$ FNR), by showing empirical evidences. In line with this established practice, our paper follows the same procedure.
However, in our setting, addressing this task presents a computational challenge due to the need for exhaustive enumeration. 
Consequently, developing a computationally efficient approach is practically important.

To the best of our knowledge, none of the existing methods efficiently control the FPR of AD results obtained by RANSAC.
\cite{tsukurimichi2022conditional} proposed a method for testing the anomalies when they are detected by LAD and Huber regressions.
In this method, controlling the FPR is achievable by relying on the KKT optimality conditions inherent in the convex optimization problems of LAD and Huber regressions, which are not present in RANSAC.
To achieve FPR control in the case of RANSAC, we need to carefully examine each step of the algorithm.
Consequently, the method proposed in \cite{tsukurimichi2022conditional} is not applicable to our setting. 
%
%Additionally, it has a drawback of high computational cost due to the need for exhaustive exploration of the data space.

To address the challenge of controlling the FPR, our idea is to leverage \emph{Selective Inference} \citep{lee2016exact} to perform statistical tests on the detected anomalies.
However, employing existing SI methods \citep{lee2016exact, chen2020valid} results in significantly high FNR. 
Although there are techniques proposed in the literature to resolve the high FNR problem \citep{le2021parametric, sugiyama2021more, tsukurimichi2022conditional, le2024cad}, introducing these methods to RANSAC-based AD is non-trivial because they are designed for specific machine learning algorithms. 
Furthermore, these methods suffer from computational drawbacks due to the need of exhaustively searching the data space.
Therefore, we need to propose a computationally efficient approach to minimize the FNR while properly controlling the FPR.

\textbf{Contributions.} Our contributions are as follows:

$\bullet$ We propose a novel statistical method, named \emph{CTRL-RANSAC}, for testing the results of RANSAC-based AD with controllable FPR.
To our knowledge, this is the first method capable of providing valid $p$-values for anomalies detected by RANSAC.

$\bullet$ We introduce an efficient approach to enhance the TPR, i.e., minimize the FNR. 
This approach employs the concept of Dynamic Programming to achieve a significant reduction in computational time.

$\bullet$ We conduct experiments on both synthetic and real-world datasets to support our theoretical results on successful FPR control, showcasing that CTRL-RANSAC achieves superior performance with higher TPR than existing methods, enhance computational efficiency, and good results in practical applications.

%\begin{figure*}[!t]
%\centering
%\includegraphics[width=.85\textwidth]{xx}
%\caption{
%}
%\label{fig:illustration}
%\end{figure*}

\begin{table}[!t]
\renewcommand{\arraystretch}{1.2}
\centering
\caption{The importance of the proposed method lies in its ability to control the False Positive Rate (FPR).}
\vspace{-5pt}
\begin{tabular}{ |l|c|c| } 
  \hline
  & $N = 120$ & $N = 240$ \\
  \hline
  \textbf{No Inference} & FPR = 1.0 & FPR = 1.0  \\
   \hline
  \textbf{Naive Inference} & 0.23 & 0.25  \\
  \hline
  \textbf{CTRL-RANSAC} & \textbf{0.04} & \textbf{0.05}  \\
  \hline
\end{tabular}
\label{tbl:example_intro}
\vspace{-10pt}
\end{table}

\begin{example}
To show the importance of the proposed CTRL-RANSAC, we present an example presented in Fig. \ref{fig:illustration}.
Given a dataset for a regression task of examining the relationship between medication dosage and blood pressure reduction in a group of patients, the goal is to detect those with abnormal reductions in blood pressure.
We applied RANSAC to identify anomalies and successfully detected two abnormal data points. 
However, the method erroneously identified two patients with normal blood pressure as anomalies.
To address this issue, we conducted an additional hypothesis testing step using the proposed $p$-values, which allowed us to successfully identify both true positive and false positive detections. 
Furthermore, we repeated the experiments $N$ times and the FPR results are shown in Tab. \ref{tbl:example_intro}.
With the proposed CTRL-RANSAC, we successfully controlled the FPR at $\alpha = 0.05$, meaning the probability of mistakenly identifying normal cases as anomalies was controlled below $5\%$, which competing methods were unable to achieve.

\end{example}

\textbf{Related works.}
In the context of regression problems under the presence of anomalies, several studies have been proposed in the literature for AD tasks using traditional statistical theories  \citep{srikantan1961testing, joshi1972some, ellenberg1973joint, ellenberg1976testing, rahmatullah2005identifying}.
Additionally, there have been investigations into testing the anomalies based on statistical hypothesis testing.
For example, in \cite{pan1995multiple} and \cite{srivastava1998outliers}, the likelihood ratio test is used to determine whether an individual is an anomaly. 
%
%These studies employ the mean-shift model which is common for characterizing the presence of anomalies.
%
Unfortunately, these classical methods assume that anomalies are pre-determined in advance.
When the anomalies are identified by AD algorithms, the assumption is violated, leading to invalid classical tests, i.e., failure to control the FPR.

Controlling the FPR when using classical inference methods requires the application of multiple testing correction.
The Bonferroni correction is a widely used technique for this purpose. 
It adjusts the pre-determined level $\alpha$ by dividing it by the correction factor, which scales exponentially to a value of $2^n$, where $n$ is the number of instances.
As $n$ increases, this exponential growth can lead to an excessively large correction factor, resulting in overly conservative statistical tests and a significantly high FNR.

In recent years, SI has emerged as a promising approach to address the invalidity of traditional inference methods. 
Unlike conventional methods that employ overly conservative correction factors of $2^n$, SI address the problem through conditional inference. 
By conditioning on a specific set of identified anomalies, SI effectively reduces the correction factor to 1.
This is the basic concept of the conditional SI introduced in the seminal work of  \cite{lee2016exact}.

After the seminal work \cite{lee2016exact}, has been actively studied and applied to various problems, such as feature selection~\citep{fithian2014optimal, tibshirani2016exact, yang2016selective, sugiyama2021more, duy2021more}, change point detection~\citep{umezu2017selective, hyun2018post, duy2020computing, sugiyama2021valid, jewell2022testing}, clustering~\citep{lee2015evaluating, inoue2017post, gao2022selective}, and segmentation~\citep{tanizaki2020computing}.
Furthermore, SI can be applied to conduct statistical inference in deep learning models \citep{duy2022quantifying, miwa2023valid, shiraishi2024statistical}

The most relevant works to this paper are \citet{chen2020valid}, \citet{tsukurimichi2022conditional}, and \citet{le2024cad}.
In \citet{chen2020valid}, the authors developed a method for testing the features of a linear model after excluding anomalies.
While their focus was \emph{not} on testing the anomalies directly, their approach inspired the work in \citet{tsukurimichi2022conditional}.
In \citet{tsukurimichi2022conditional}, the authors introduced a SI approach for testing outliers identified by LAD and Huber regressions, but the reliance on KKT conditions of the convex optimization problems limits its applicability.
%
%However, their approach depends on the KKT conditions of the convex optimizations, which limits its applicability.
%
Consequently, their method is unsuitable for heuristic algorithms such as RANSAC.

\citet{le2024cad} introduced an approach for testing anomalies in the context of domain adaptation.
However, their focus on unsupervised AD that differs from the setting we consider in this paper, which is AD in the context of regression problems.
Additionally, if we leverage and apply the ideas of the methods proposed in \citet{tsukurimichi2022conditional} and \citet{le2024cad} to our problem, we would still encounter computational challenges due to the need for exhaustive searches throughout the data space.

%% file: sec2.tex
\section{Problem Setup} \label{sec:problem_setup}

To formulate the problem, we consider a regression setup with a feature matrix $X \in \RR^{n \times p}$ and a random response vector defined by
\begin{equation} \label{eq:random_vector}
	{\bm Y} = (Y_1, ..., Y_n)^\top \sim \NN({\bm \mu}, \Sigma),
\end{equation}
where $n$ is the number of instances, $p$ is the number of features, ${\bm \mu}$ is the unknown mean vector, and $\Sigma \in \RR^{n \times n}$ is a covariance matrix, which is known or estimable from independent data.
The feature matrix $X$ is assumed to be non-random.
The goal is to conduct statistical tests on the anomalies obtained by applying the RANSAC algorithm to the data.

\subsection{RANSAC-based Anomaly Detection}

The procedure of detecting anomalies using RANSAC \citep{fischler1981random} is as follows:

\textbf{Step 1 (fit a regression model with RANSAC)}:

$\bm a)~$ \textbf{Iterative repetition}. 
For $b = 1$ to $B$:
\begin{itemize}
%$\quad~~  i)$ 
	\item[i.] Randomly select a subset from the original data:
\begin{align*}
	s^{(b)} \subseteq [n] = \{1, 2, ..., n\}.
\end{align*}
%
%$\quad~  ii)$ 
	\item[ii.] Fit the regression model with the selected subset:
\begin{align} \label{eq:b_iter_model_fitting}
	\hat{\bm \beta}^{(b)} = \argmin \limits_{\bm \beta^{(b)} \in \RR^p}
	\frac{1}{2} 
	\left | \left |
	\bm Y_{s^{(b)}} - X_{s^{(b)}}
	\bm \beta^{(b)}
	\right |\right|_2^2,
\end{align}
%	where $X_{s^{(b)}}$ is a sub-matrix of $X$ made up of rows in the set $s^{(b)}$ and $\bm Y_{s^{(b)}}$ is a sub-vector of $\bm Y$	
where $X_{s^{(b)}}$ and $\bm Y_{s^{(b)}}$ are the sub-matrix of $X$ and sub-vector of $\bm Y$ made up of rows in the set $s^{(b)}$.

	\item[iii.] Determine a set of \emph{inliers} (i.e., consensus set):
%$\quad  iii)$
%
\begin{align} \label{eq:select_inliers}
    \cI^{(b)} = 
    \Big\{ 
            i \in [n]  \; \big| \;  
            \left(
                \bm Y_i - X_i^\top \hat{\bm \beta}^{(b)}
            \right)^2  
            \leq 
            \tau
    \Big\},
\end{align}
where $\tau$ is a pre-determined threshold.
%
%We define
% %
% \begin{align*}
% 	\cO^{(b)} = [n] \setminus \cI^{(b)} 
% \end{align*}
%as the set of anomalies at the $b^{\rm th}$ iteration.
\end{itemize}

$\bm b)~$ \textbf{Selecting the ``optimal'' model.} 
After the fixed $B$ number of trials, the optimal model is the one that has the largest consensus set, i.e., for $b \in [B]$,
\begin{align}\label{eq:optimal_consensus_set}
	\cI  = \argmax \limits_{\cI^{(b)}}
	\big |\cI^{(b)} \big |.
\end{align}
%
%There are two common approaches for selecting the optimal model. The first approach is to choose the model with the optimal set of inliers, i.e., the model where the fitting error on the corresponding set of inliers is smallest:
%%
%\begin{align*}
%	\cI  = \argmin \limits_{\cI^{(b)}}
%	\left | \left |
%	\bm Y_{\cI^{(b)}} - X_{\cI^{(b)}}
%	\hat{\bm \beta}^{(b)}
%	\right |\right|_2^2
%\end{align*}
%for $b \in [B]$.
%%
%The second approach is to select the model that has the largest set of inliers, i.e., for $b \in [B]$,
%\begin{align*}
%	\cI  = \argmax \limits_{\cI^{(b)}}
%	\big |\cI^{(b)} \big |.
%\end{align*}
%
In case of multiple models have the same largest set of inliers, the first encountered model is selected.

\textbf{Step 2 (identifying anomalies)}: after identifying the optimal set of inliers $\cI$ in step 1, we obtain a set $\cO$ of indices of anomalies:
\begin{align} \label{eq:detected_anomalies}
	\cO = [n] \setminus \cI.
\end{align}

\subsection{Statistical Inference and Decision Making}

Our goal is to assess if the anomalies in \eq{eq:detected_anomalies} are truly abnormal. 
We consider testing the detected anomalies using the model after removing the anomalies.

\paragraph{The null and alternative hypotheses.} To achieve the task, we consider the following statistical test:
\begin{align} \label{eq:hypotheses}
    	{\rm H}_{0, i} : \bm \mu_i = X^{\top}_i \bm \beta_{-\cO}
   	\quad
	\text{vs.}
	\quad
    	{\rm H}_{1, i} : \bm \mu_i \neq X^{\top}_i \bm \beta_{-\cO},
\end{align}
$\forall i \in \cO$.
The $\bm \beta_{-\cO}$ is the population least square estimate after removing the set of anomalies:
\begin{align*}
	\bm \beta_{-\cO} = 
	\left ( X_{- \cO} \right )^{+} \bm \mu_{- \cO},
\end{align*}
where
%$
%(X_{- \cO})^{+} = 
%\left (
%X_{- \cO}^\top  
%X_{- \cO}
%\right )^{-1} X_{- \cO}^\top
%$,
$(X_{- \cO})^{+}$ is the 
\emph{pseudo-inverse} of $X_{- \cO}$,
\begin{align*}
X_{- \cO} = I^n_{-\cO} X 
\quad \text{and} \quad  \bm \mu_{- \cO} = I^n_{-\cO} \bm \mu,
\end{align*}
with $I^n_{-\cO} \in \RR^{n \times n}$ is a diagonal matrix in which the $i^{\rm th}$ diagonal element set to $0$ if $i \in \cO$, and $1$ otherwise.

In other words, we aim to test if each of the detected anomalies $i \in \cO$ is truly deviated from the model fitted after the set of anomalies is removed.
The above hypotheses was also used in \citet{tsukurimichi2022conditional}

\paragraph{Test statistic.} In order to test the hypotheses, the test statistic is defined as follows:
\begin{align}\label{eq:test_statistic}
	T_i = \bm Y_i - X_i^\top \hat{\bm \beta}_{- \cO},
\end{align}
where $\hat{\bm \beta}_{- \cO} = \left ( X_{- \cO} \right )^{+} \bm Y_{- \cO}$. 
We can re-write the test statistic in the form of a linear contrast w.r.t $\bm Y$:
\begin{align} \label{eq:test_statistic_linear_contrast}
	T_i = \bm \eta_i^\top \bm Y,
\end{align}
where $\bm \eta_i$ is the direction of the test statistic
\begin{equation} \label{eq:test_direction}
    \bm \eta_i = 
    \Big(
        \bm e^{\top}_i  
        - 
        X^\top_i \big(X_{-\cO}\big)^{+} I_{-\cO}^n
    \Big)^{\top},
\end{equation}
$\bm e_i \in \mathbb{R}^{n}$ 
is a basis vector with a 1 at position \textit{i}\textsuperscript{th}.

%
% $I^{-\mathcal{O}^{obs}}_{n}$ indicates an n-by-n diagonal matrix in which the \textit{j}\textsuperscript{th} diagonal entry is 0 if $j \in \mathcal{O}^{obs}$, and 1 otherwise. Now we can rewrite a test statistic (\ref{defaultTestStatistic}) in terms of $\boldsymbol{\eta}_{j}$:
%\begin{equation} \label{transformedTestStatistic}
%    Z = \boldsymbol{\eta}^{\top}_{j} \boldsymbol{y}.
%\end{equation}
%Furthermore, the statistical test (\ref{defaultStatisticalTest}) can be transformed to:
%\begin{equation} 
%    \mathbf{H}_{0, j} : \boldsymbol{\eta}^{\top}_{j}\boldsymbol{\mu} = 0,
%    \qquad
%    \mathbf{H}_{1, j} : \boldsymbol{\eta}^{\top}_{j}\boldsymbol{\mu} \neq 0.
%\end{equation}

\paragraph{Compute $p$-value and decision making.} After obtaining the test statistic from \eq{eq:test_statistic_linear_contrast}, we need to compute the $p$-value.
Given a significance level $\alpha \in [0, 1]$, typically set at 0.05, we reject the null hypothesis ${\rm H}_{0, i}$ and  and conclude that $\bm Y_i$ is an anomaly if the corresponding 
$p$-value is less than or equal to $\alpha$.
On the other hand, if the $p$-value exceeds $\alpha$, we fail to reject the null hypothesis and conclude that there is not sufficient evidence to assert that $\bm Y_i$ is an anomaly.

\paragraph{Difficulty in computing a valid $p$-value.}
The conventional (naive) $p$-value is defined as: 
\begin{align*}
	p^{\rm naive}_i = 
	\mathbb{P}_{{\rm H}_{0, i}} 
	\Big ( 
		\left | \bm \eta_i^\top \bm Y \right |
		\geq 
		\left | \bm \eta_i^\top \bm Y^{\rm obs}  \right |
	\Big ), 
\end{align*}
where $\bm Y^{\rm obs}$ is an observation (realization) of the random vector $\bm Y$.
If the hypotheses in \eq{eq:hypotheses} are fixed in advanced, i.e., non-random, the vector $\bm \eta_i$ is independent of the data and RANSAC algorithm.
Thus, the naive $p$-value is valid in the sense that 
\begin{align} \label{eq:valid_p_value}
	\mathbb{P} \Big (
	\underbrace{p_i^{\rm naive} \leq \alpha \mid {\rm H}_{0, i} \text{ is true }}_{\text{a false positive}}
	\Big) = \alpha, ~~ \forall \alpha \in [0, 1],
\end{align} 
i.e., the probability of obtaining a false positive is controlled under the pre-specified level $\alpha$.
However, in our setting, the hypotheses are actually \emph{not} fixed; they are defined based on the data and RANSAC algorithm.
As a result, the property of a valid $p$-value in \eq{eq:valid_p_value} is no longer satisfied.
Hence, the naive $p$-value is \emph{invalid}.

%% file: sec3.tex
\section{PROPOSED CTRL-RANSAC
}
\label{sec:method}

In this section, we present the details of the proposed CTRL-RANSAC method. 
We introduce valid $p$-values for anomalies detected by RANSAC by leveraging the SI framework \citep{lee2016exact}. 
Additionally, we propose an efficient approach to calculate these $p$-values based on the concept of Dynamic Programming.

\subsection{The valid \textit{p}-value in CTRL-RANSAC}

In order to compute the valid \textit{p}-value, we first need to identify the distribution of the test statistic in \eq{eq:test_statistic_linear_contrast}.
We accomplish this by employing the concept of SI, specifically by examining the sampling distribution of the test statistic \emph{conditional} on the AD results:
\begin{align} \label{eq:conditional_distribution}
	\bP
	\Big ( 
	\bm \eta_i^\top \bm Y
	\mid 
	\cO(\bm Y) = \cO^{\rm obs}
	\Big ),
\end{align}
where $\cO(\bm Y) \subseteq [n]$ indicates the set of anomalies identified by applying RANSAC to any random vector $\bm Y$, and $\cO^{\rm obs} = \cO(\bm Y^{\rm obs})$ is the set of anomalies obtained from an observation $\bm Y^{\rm obs}$ of the random $\bm Y$.

Based on the distribution of the test statistic in \eq{eq:conditional_distribution}, we introduce the \emph{selective $p$-value} defined as:
\begin{align}\label{eq:selective_p}
	p_i^{\rm selective} = 
	\bP_{{\rm H}_{0, i}} 
	\Big ( 
		\left |\bm \eta_i^\top \bm Y \right |
		\geq
		\left |\bm \eta_i^\top \bm Y^{\rm obs} \right |
		~
		\Big | 
		~\cE 
	\Big ),
\end{align}
where $\cE$ is the conditioning event defined as:
\begin{align} \label{eq:cE}
	\cE = 
	\left \{
		\cO(\bm Y) = \cO^{\rm obs},
		\cQ(\bm Y) = \cQ^{\rm obs}
	\right \}.
\end{align}
The $\cQ(\bm Y)$ is the sufficient statistic of the \emph{nuisance component} defined as:
\begin{equation} \label{eq:nuisance}
    \cQ(\bm Y) := \big( I^n -\bm  b \bm \eta^\top_i \big) \bm Y,
\end{equation}
where $\bm b = \Sigma \bm \eta_i(\bm \eta^\top_i \Sigma \bm \eta_i)^{-1}$ and $\cQ^{\rm obs} = \cQ(\bm Y^{\rm obs})$.

%$\bm b = \frac{\Sigma \bm \eta_i}{\bm \eta^\top_i \Sigma \bm \eta_i}$

\begin{remark}
The nuisance component $\cQ(\bm Y)$ corresponds to the component $\bm z$ in the seminal work by \citet{lee2016exact} (see Sec. 5, Eq. (5.2), and Theorem 5.2). The additional conditioning on $\cQ(\bm Y)$ is required for technical reasons, specifically to facilitate tractable inference. This is the standard approach in SI literature and is used in almost all the cited SI-related works.
\end{remark}

\begin{lemma} \label{lemma:valid_selective_p}
The selective $p$-value in \eq{eq:selective_p} is a valid $p$-value that satisfies the following property:
\begin{align*}
	\bP_{{\rm H}_{0, i}}  \Big (
	p_i^{\rm selective} \leq \alpha
	\Big) = \alpha, ~~ \forall \alpha \in [0, 1].
\end{align*} 
\end{lemma}

\begin{proof}
The proof is deferred to Appendix \ref{appendix:valid_p_proof}.
\end{proof}

Lemma \ref{lemma:valid_selective_p} indicates that, by using the selective $p$-value, the false positive rate (FPR) is theoretically controlled for any pre-specified level of guarantee $\alpha \in [0, 1]$.

%Lemma \ref{appendix:valid_p_proof} demonstrates that the selective p-value allows us to manage the False Positive Rate (FPR) precisely at the significance level $\alpha$. Therefore, the selective \textit{p}-value is confirmed to be valid.

\subsection{Conditional Data Space Characterization} 
To compute the selective $p$-value, we need to identify the conditional data space $\cY$ that corresponds to the conditioning event in \eq{eq:cE}, i.e.
\begin{align} \label{eq:conditional_data_space}
	\cY = 
	\big \{ 
		\bm Y \in \RR^n
		\mid 
		\cO(\bm Y) = \cO^{\rm obs},
		\cQ(\bm Y) = \cQ^{\rm obs}
	\big \}
\end{align}
In the following lemma, we show that the conditional data space $\cY$ is, in fact,  restricted to a \emph{line} in $\RR^n$.
% due to the second condition on the nuisance component $\cQ(\bm Y)$

\begin{lemma} \label{lemma:data_line}
The subspace $\cY$ in \eq{eq:conditional_data_space} is restricted to a line parametrized by a scalar parameter $z \in \RR$:
\begin{align} \label{eq:parametrized_response_vector}
	\cY = \big \{ 
		\bm Y (z) = \bm a + \bm b z \mid z \in \cZ
	\big \},
\end{align}
where $\bm a = \cQ^{\rm obs}$, $\bm b$ is defined in \eq{eq:nuisance}, and 
\begin{align}\label{eq:cZ}
	\cZ = \big \{ z \in \RR \mid \cO(\bm a + \bm b z) = \cO^{\rm obs} \big \} 
\end{align}

\end{lemma}

\begin{proof}
The proof is deferred to Appendix \ref{appendix:data_line_proof}.
\end{proof}

\begin{remark} The fact that the conditional space can be reduced to a line was implicitly exploited in \cite{lee2016exact} and discussed in Sec. 6 of \cite{liu2018more}.
Lemma \ref{lemma:data_line} shows that it is not necessary to consider the $n$-dimensional space.
Instead, we only need to focus on the one-dimensional projected data space $\cZ$ in \eq{eq:cZ}.
\end{remark}

\paragraph{Reformulation of the selective \textit{p}-value computation with $\cZ$.}
Let us denote a random variable $Z \in \RR$ and its observation $Z^{\rm obs} \in \RR$ as follows:
\begin{align*}
	Z = \bm \eta_i^\top \bm Y
	~~ \text{and} ~~ 
	Z^{\rm obs} = \bm \eta_i^\top \bm Y^{\rm obs}.
\end{align*}
%
%The test statistic in \eq{eq:test_statistic_linear_contrast} follows a \emph{truncated} normal (TN) distribution, i.e,
%\begin{equation*}
%    Z \mid Z \in \mathcal{Z} \sim {\rm TN} \big( \boldsymbol{\eta}_{i}^{\top}\boldsymbol{\mu}, \boldsymbol{\eta}^{\top}_{i}\Sigma\boldsymbol{\eta}_{i}, \mathcal{Z} \big),
%\end{equation*}
%
We have 
$
Z \mid Z \in \mathcal{Z} \sim {\rm TN} \big( \boldsymbol{\eta}_{i}^{\top}\boldsymbol{\mu}, \boldsymbol{\eta}^{\top}_{i}\Sigma\boldsymbol{\eta}_{i}, \mathcal{Z} \big)
$,
which is the \emph{truncated} normal (TN) distribution with mean $\boldsymbol{\eta}^{\top}_{i} \boldsymbol{\mu}$, variance $\boldsymbol{\eta}^{\top}_{i}\Sigma\boldsymbol{\eta}_{i}$, and truncation region $\mathcal{Z}$. 
Then, the selective \textit{p}-value in \eq{eq:selective_p} can be rewritten
as follows:
\begin{equation} \label{eq:selective_p_reformulated}
    p^{\rm selective}_{i} = \mathbb{P}_{{\rm H}_{0,i}} \Big( |Z| \geq |Z^{\rm obs}| \; \Big| \; Z \in \mathcal{Z} \Big). 
\end{equation}
The above equation shows that the remaining requirement to compute $p^{\rm selective}_{i}$ is to obtain the truncation region $\mathcal{Z}$. In the following section, we will propose a computationally efficient method for identifying $\mathcal{Z}$.

\begin{figure*} \label{DetectingAndTestingFig}
    \includegraphics[width=1\linewidth]{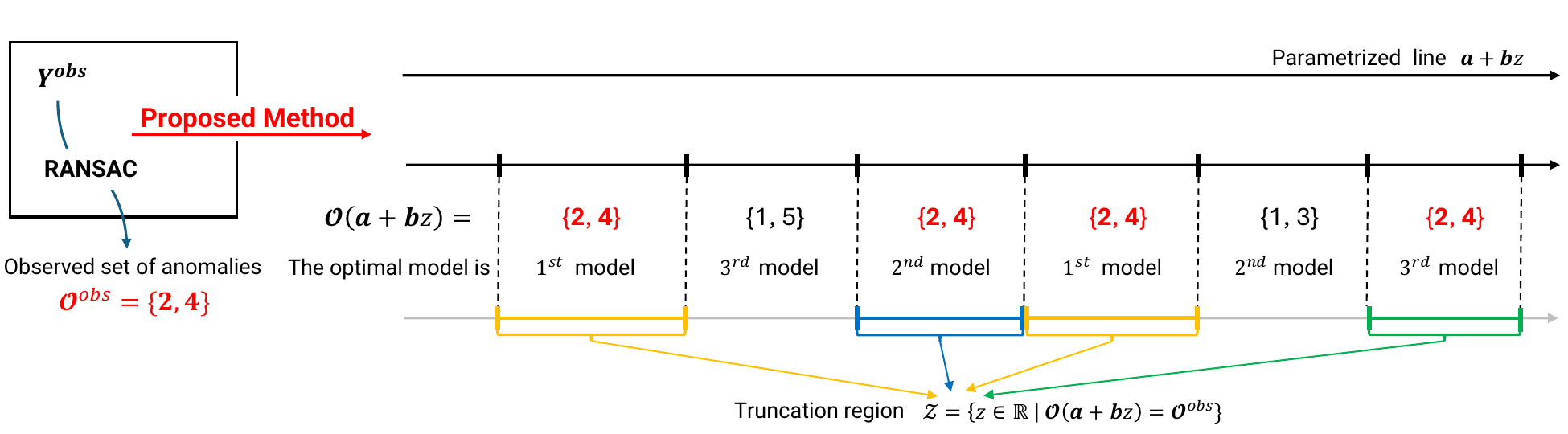} 
    \caption{A schematic illustration of the proposed method. By applying RANSAC to the observed data, we obtain a set of anomalies. Then, we parametrize the data with a scalar parameter $z$ in the direction of the test statistic to identify the truncation region $\cZ$ whose data have the \textit{same} result of AD as the observed data. Finally, the inference is conducted conditional on $\cZ$.
    We employ the ``divide-and-conquer'' strategy and introduce an efficient method for characterizing the truncation region $\cZ$.}
    \label{fig:schematic_illustration}
    \vspace{-5pt}
\end{figure*}

\subsection{Identification of Truncation Region $\cZ$}

Given the complexity of $\cZ$, there is no direct way to identify it.
To resolve this challenge, we utilize the \emph{``divide-and-conquer''} strategy and propose a method (illustrated in Fig. \ref{fig:schematic_illustration}) to efficiently identify $\cZ$:

$\bullet$ We divide the problem into multiple sub-problems

$\bullet$ We introduce a dynamic programming (DP) approach to efficiently solve each sub-problem

$\bullet$ We combine multiple sub-problems to obtain $\cZ$

\textbf{Divide-and-conquer strategy.}
The $\cZ$ in \eq{eq:cZ} can be decomposed into multiple sub-problems as follows:
\begin{align}
	\hspace{-2.5mm}
	\cZ &= \big \{ z \in \RR \mid \cO(\bm a + \bm b z) = \cO^{\rm obs} \big \} \nonumber \\
	&= \bigcup \limits_{ b \in [B]}
	\underbrace{\Bigg \{
	z \in \RR ~\Big |
	\begin{array}{l}
	$\text{the $b^{\rm th}$ model is optimal,}$ \\ 
	\cO^{(b)}(\bm a + \bm b z) = \cO^{\rm obs} 
	\end{array}
	\Bigg \}}
	_{\text{a sub-problem}}, \label{eq:decomposed_cZ}
\end{align}
where $\cO^{(b)}(\bm a + \bm b z)$ indicates the set of outliers detected at the $b^{\rm th}$ iteration, when applying RANSAC to $\bm Y (z) = \bm a + \bm b z$.

\paragraph{Solving of each sub-problem.}
For $b \in [B]$, we define the subset of one-dimensional projected dataset on a line for the sub-problem as:
\begin{align} \label{eq:sub_problem}
	\cZ^{(b)} = 
	\Bigg \{
	z \in \RR ~\Big |
	\begin{array}{l}
	\text{the $b^{\rm th}$ model is optimal,} \\ 
	\cO^{(b)}(\bm a + \bm b z) = \cO^{\rm obs} 
	\end{array}
	\Bigg \}.
\end{align}
In fact, the $\cZ^{(b)}$ in \eq{eq:sub_problem} is equivalent to:
\begin{align*} %\label{eq:rewrite_sub_problem}
	\cZ^{(b)} = 
	\Bigg \{
	z \in \RR ~\Bigg |
	\begin{array}{l}
        \text{the $b^{\rm th}$ model is optimal,} \\
        \big| \cO^{(b)}(\bm a + \bm b z) \big| = \big| \cO^{\rm obs} \big|, \\ 
	\cO^{(b)}(\bm a + \bm b z) = \cO^{\rm obs} 
	\end{array}
	\Bigg \}.
\end{align*}
Then, the $\cZ^{(b)}$ can be re-written as 
$
\cZ^{(b)} = \cZ^{(b)}_1 \cap \cZ^{(b)}_2,
$
%where 
%
\begin{equation*} %\label{eq:cZ1_cZ2}
\begin{aligned}
	\cZ^{(b)}_1 
	& =  
	\Bigg \{
	z \in \RR ~\Big |
	\begin{array}{l}
	\text{the $b^{\rm th}$ model is optimal,} \\ 
	\big| \cO^{(b)}(\bm a + \bm b z) \big| = \big| \cO^{\rm obs} \big| 
	\end{array}
	\Bigg \}, \\
	\cZ^{(b)}_2 
	& =  
	\{
	z \in \RR \mid
	\cO^{(b)}(\bm a + \bm b z) = \cO^{\rm obs}
	\}.
\end{aligned}
\end{equation*}
Since the explanation for identifying $\cZ^{(b)}_2$ is simpler than that of $\cZ^{(b)}_1$, we will present it first, followed by the identification of $\cZ^{(b)}_1$.

\begin{lemma}\label{lemma:cZ_b_2}
The set $\cZ^{(b)}_2$ can be computed by solving quadratic inequalities w.r.t. $z$ described as follows:
\begin{align*}
	\cZ^{(b)}_2 = 
	\Big ( 
	\cap_{i \not \in \cO^{\rm obs}} \cR^{(b)}_i
	\Big )
	\cap
	\Big ( 
	\cap_{i \in \cO^{\rm obs}} \overline{\cR}^{(b)}_i
	\Big ),
\end{align*}
where $\cR^{(b)}_i$ is the region of the parameter $z$ where the $i^{\rm th}$ data point, $i \in [n]$, is NOT an outlier (i.e., is an inlier) at the $b^{\rm th}$ iteration of RANSAC, defined as:
\begin{align} \label{eq:cR_b_i}
	\cR^{(b)}_i = 
	\Big \{ 
		z \in \RR \mid 
		\underbrace{\left(
                	\bm Y_i (z) - X_i^\top \hat{\bm \beta}^{(b)} (z)
	         \right)^2  
	         \leq 
                 \tau}_{\text{a quadratic inequality w.r.t }z}
	\Big \},
\end{align}
and $\overline{\cR}^{(b)}_i$ is the complement of  $\cR^{(b)}_i$.
In \eq{eq:cR_b_i},
$\bm Y_i (z) = (\bm a + \bm b z)_i$ and 
$\hat{\bm \beta}^{(b)} (z)$ is the least square estimate when solving \eq{eq:b_iter_model_fitting} on $\bm Y(z)$, which is a linear function of $z$.
\end{lemma}

%The proof is deferred to Appendix \ref{app:proof_cZ_b_2}.

\begin{proof}
The proof is deferred to Appendix \ref{app:proof_cZ_b_2}.
\end{proof}

%In the above lemma, the region $\cR^{(b)}_i$ in \eq{eq:cR_b_i} can be computed by solving a \emph{quadratic inequality} w.r.t $z$ because $\bm Y_i (z)$ and $\hat{\bm \beta}^{(b)} (z)$ are \emph{linear functions} of $z$.
%
\textbf{Interpretation of  Lemma \ref{lemma:cZ_b_2}.} When obtaining $\cO^{\rm obs}$, we have already known which data points have been detected as abnormal. 
Then, for any $i \not \in \cO^{\rm obs}$,  we determine the region of $z$ where the $i^{\rm th}$ data point is NOT an outlier at the $b^{\rm th}$ iteration of RANSAC. 
Similarly, for $i \in \cO^{\rm obs}$, we find the region of $z$ where the $i^{\rm th}$ data point is an outlier.
Finally, by intersecting all the identified regions, we obtain $\cZ^{(b)}_2$.

%
%\textbf{Identification of $\cZ^{(b)}_2$.} 
%At each $b^{\rm th}$ iteration, we need to determine if the $i^{\rm th}$ data point is inlier or not based on \eq{eq:select_inliers}.
%%
%Let us define the region of $z$ where the $i^{\rm th}$ data point is the inlier at the $b^{\rm th}$ iteration of RANSAC, applied to the parametrized data $\bm Y(z) = \bm a + \bm b z$, as:
%%
%\begin{align} \label{eq:cR_b_i}
%	\cR^{(b)}_i = 
%	\Big \{ 
%		z \in \RR \mid 
%		\underbrace{\left(
%                	\bm Y_i (z) - X_i^\top \hat{\bm \beta}^{(b)} (z)
%	         \right)^2  
%	         \leq 
%                 \tau}_{\text{a quadratic inequality w.r.t }z}
%	\Big \},
%\end{align}
%%
%where $\hat{\bm \beta}^{(b)} (z)$ is simply the least square estimate when solving \eq{eq:b_iter_model_fitting} on $\bm Y(z)$.

\begin{lemma} \label{lemma:cZ_b_1}
The set $\cZ^{(b)}_1$ can be re-written as follows:
\begin{align*}
	\cZ^{(b)}_1
    = 
    \Bigg( \bigcap_{u \in [B], u\neq b} \cS^{(u)}_{n, |\cO^{\rm obs}|}\Bigg) \bigcap \Bigg( \cS^{(b)}_{n,|\cO^{\rm obs}| - 1} \setminus\cS^{(b)}_{n, |\cO^{\rm obs}|} \Bigg)
\end{align*}
where $\cS^{(u)}_{n, |\cO^{\rm obs}|}$ is  defined as:
\begin{align} \label{eq:cS_ub_n}
	\cS^{(u)}_{n, |\cO^{\rm obs}|} = 
	\left \{ 
		z \in \RR :
		\big | \cO^{(u)}_n (z)
		\big |
		>
		\big | \cO^{\rm obs} \big|
	\right \},
\end{align}
$\cO^{(u)}_n (z)$ is the set of outliers whose indices are smaller than or equal to $n$, detected by the $u^{\rm th}$ model when applying RANSAC to the vector $\bm Y (z) = \bm a + \bm b z$.
\end{lemma}

%The proof is deferred to Appendix \ref{app:proof_cZ_b_1}.

\begin{proof}
The proof is deferred to Appendix \ref{app:proof_cZ_b_1}.
\end{proof}

\textbf{Interpretation of Lemma \ref{lemma:cZ_b_1}.} The $\cS^{(u)}_{n, |O^{\rm obs}|}$ in \eq{eq:cS_ub_n} is the region of $z$ where the outlier set detected by the $u^{\rm th}$ model is larger than the observed outlier set.
By identifying all such regions $\cS^{(u)}_{n, |O^{\rm obs}|}$ for any $u \in [B], u \neq b$, and taking their intersection, we obtain the region where the outlier set detected by every model except the $b^{\rm th}$ model is larger than the observed set. Further intersecting this region with $\Big( \cS^{(b)}_{n,|\cO^{\rm obs}| - 1} \setminus\cS^{(b)}_{n, |\cO^{\rm obs}|} \Big)$, which restricts the $b^{\rm th}$ model to identify an outlier set of the same size as the observed one, we obtain region $\cZ^{(b)}_1$. In this region, the $b^{\rm th}$ model is optimal and $\big| \cO^{(b)}(\bm a + \bm b z) \big| = \big| \cO^{\rm obs} \big|$.
%
%, i.e., the outlier set of the $b^{\rm th}$ model is the smallest.

\begin{remark}
Here, we assume that only one model among the $B$ models has the smallest set of outliers (i.e., the largest consensus set).
The extension of Lemma \ref{lemma:cZ_b_1} to the case where multiple models share the smallest outlier set is provided in Appendix \ref{app:extension_lemma_cZ_1}.
\end{remark}

%\paragraph{Computation of $\cR_i^{(b)}$ in \eq{eq:cR_b_i}.} 
To obtain $\cZ^{(b)}_1$, the remaining task is to efficiently compute $\cS^{(u)}_{n, |\cO^{\rm obs}|}$ in \eq{eq:cS_ub_n}.
In the next Lemma, we show that $\cS^{(u)}_{n, |\cO^{\rm obs}|}$ can be computed in a \emph{recursive} manner.

\begin{lemma} \label{lemma:recursive}
Let us define the recursive formula as:
\begin{align*}
\begin{footnotesize}
\cS^{(u)}_{j, k} =
\begin{cases}
    \emptyset & \text{if } j \leq k, \\
    \overline{\cR}^{(u)}_j \cup \cS^{(u)}_{j - 1, k} & \text{if } k = 0, \\
    \Big( \cR^{(u)}_j \cap \cS^{(u)}_{j - 1, k} \Big) \cup \Big( \overline{\cR}^{(u)}_j \cap \cS^{(u)}_{j - 1, k - 1} \Big) & \text{otherwise}.
    \end{cases}
 \end{footnotesize}
\end{align*}
Then, $\cS^{(u)}_{n, |\cO^{\rm obs}|}$ can be computed by setting $j = n$ and $k = |\cO^{\rm obs}|$ in the above recursive formula.
\end{lemma}

\begin{proof}
The proof is deferred to Appendix \ref{app:proof_recursive}.
\end{proof}

%The proof is deferred to Appendix \ref{app:proof_recursive}.
%
Since $\cS^{(u)}_{n, |\cO^{\rm obs}|}$ can be obtained recursively, we employ the \emph{bottom-up dynamic programming} approach to efficiently compute $\cS^{(u)}_{n, |\cO^{\rm obs}|}$, which subsequently used to identify $\cZ_1^{(b)}$ in Lemma \ref{lemma:cZ_b_1}.

After obtaining $\cZ_2^{(b)}$ and $\cZ_1^{(b)}$ from Lemmas \ref{lemma:cZ_b_2} and \ref{lemma:cZ_b_1}, the region $\cZ^{(b)}$ in \eq{eq:sub_problem} corresponding to the sub-problem is identified as $\cZ^{(b)} = \cZ^{(b)}_1 \cap \cZ^{(b)}_2$. 
Finally, by taking the union of the computed $\cZ^{(b)}$ for all $b \in [B]$, we obtain $\cZ$ in \eq{eq:decomposed_cZ},  which is subsequently used to compute the proposed selective $p$-value in \eq{eq:selective_p_reformulated}.
The entire steps of the proposed CTRL-RANSAC method is summarized in Algorithm \ref{alg:ctrl_ransac}.
The complexity analysis of Algorithm \ref{alg:ctrl_ransac} is provided in Appendix \ref{app:complexity_analysis}.

\begin{algorithm}[!t]
\caption{\texttt{CTRL-RANSAC}}
\label{alg:ctrl_ransac}
\begin{footnotesize}
\textbf{Input:} $(X, \boldsymbol{Y}^{\rm obs})$
\begin{algorithmic}[1]
\vspace{2pt}
    \STATE $\mathcal{O}^{\rm obs} \gets$ Apply RANSAC to $(X, \boldsymbol{Y}^{\rm obs})$
    \vspace{2pt}
    \FOR{$i \in \mathcal{O}^{\rm obs}$}
    \vspace{2pt}
        \STATE Compute $\bm \eta_{i} \gets$ Eq. (\ref{eq:test_direction}), $\boldsymbol{a} \text{ and } \boldsymbol{b} \gets$ Eq. (\ref{eq:parametrized_response_vector})
        \vspace{2pt}
        \STATE $\mathcal{Z} \gets$ {\tt identify\_truncation\_region} ($\boldsymbol{a}, \boldsymbol{b}, \mathcal{O}^{\rm obs}$)
        \vspace{2pt}
        \STATE Compute $p^{\rm selective}_{i} \gets$ Eq. (\ref{eq:selective_p_reformulated}) with $\cZ$
        \vspace{2pt}
    \ENDFOR
\end{algorithmic}
\textbf{Output:} $\{p^{\rm selective}_{i}\}_{i \in \mathcal{O}^{\rm obs}}$
\end{footnotesize}
\end{algorithm}

\begin{algorithm}[!t]
\renewcommand{\algorithmicrequire}{\textbf{Input:}}
\renewcommand{\algorithmicensure}{\textbf{Output:}}
\begin{footnotesize}
\begin{algorithmic}[1]
\REQUIRE $\bm a, \bm b, \cO^{\rm obs}$
\vspace{2pt}
\STATE $\bm Y(z) \gets \bm a + \bm b z$
\vspace{2pt}
\FOR {$b \in [B]$}
\vspace{2pt}
\STATE $\cZ_2^{(b)} \gets $ Lemma \ref{lemma:cZ_b_2}
\vspace{2pt}
\STATE $\cZ_1^{(b)} \gets $ Lemma \ref{lemma:cZ_b_1}
\vspace{2pt}
\STATE $\cZ^{(b)} = \cZ^{(b)}_1 \cap \cZ^{(b)}_2$ // Eq. \eq{eq:sub_problem}
\vspace{2pt}
\ENDFOR
\vspace{2pt}
\STATE $\cZ \gets \cup_{ b \in [B]} \cZ^{(b)}$ // Eq. \eq{eq:decomposed_cZ}
\vspace{2pt}
\ENSURE $\cZ$ 
\end{algorithmic}
\end{footnotesize}
\caption{{\tt identify\_truncation\_region}}
\label{alg:identify_truncation_region}
\end{algorithm}

%% file: sec4.tex
\section{Experiment} \label{sec:experment}

In this section, we demonstrate the performance of the proposed CTRL-RANSAC.
Here, we present the main results. Several additional experiments can be found in Appendix \ref{app:additional_experiment}.
We considered the following methods:

%In this section, we compare the performance of our proposed method with several other approaches. Performance is assessed based on the FPR and TPR. The following methods are analyzed: 

$\bullet$ \texttt{CTRL-RANSAC}: proposed method

$\bullet$ \texttt{Line Search}: we initially proposed this method for RANSAC-based AD, drawing on the ideas of \citet{le2021parametric} and \citet{le2024cad}, with further details provided in Appendix \ref{app:line_search_method}. Later, its drawback of highly computational cost was addressed by the {\tt CTRL-RANSAC} method introduced in this paper.

$\bullet$ \texttt{OC}: the extension of \citet{lee2016exact} to our setting

$\bullet$ \texttt{Bonferroni}:  the most popular multiple testing

$\bullet$ \texttt{Naive}: traditional statistical inference

$\bullet$ \texttt{No Inference}: RANSAC without inference

If a method fails to control the FPR at  $\alpha$, it is \textit{invalid}, and its TPR becomes irrelevant. We set $\alpha = 0.05$. 
We remind that a method has high TPR indicates that it has low FNR.
We executed the code on AMD Ryzen 5 6600H with Radeon Graphics 3.30 GHz.

\subsection{Numerical Experiments}

\begin{figure}[!t]
    \centering
    \begin{subfigure}[t]{0.24\textwidth}
        \centering
        \includegraphics[width=\textwidth]{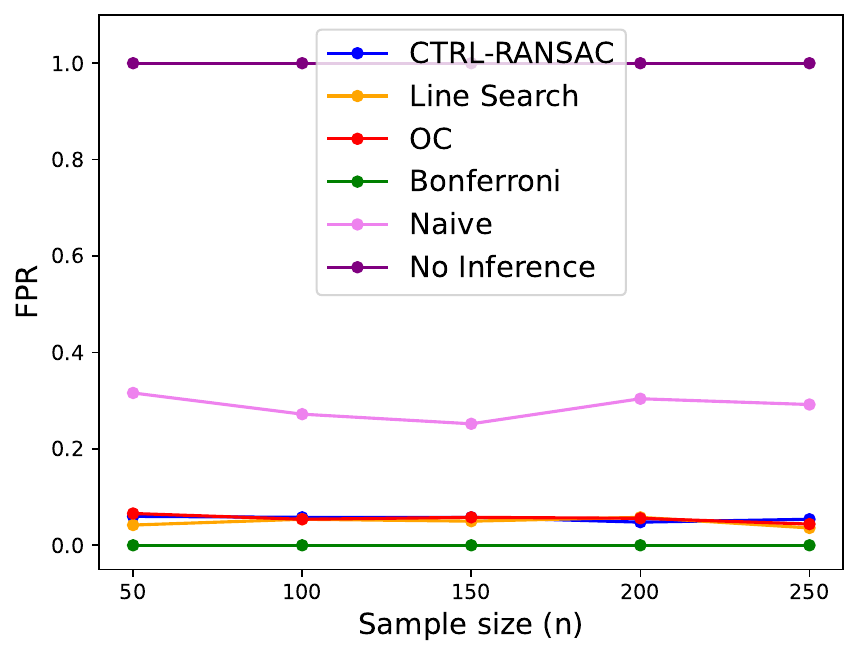}
        \caption{Independence}
    \end{subfigure}
    \hspace{-2mm}
    \begin{subfigure}[t]{0.24\textwidth}
        \centering
        \includegraphics[width=\textwidth]{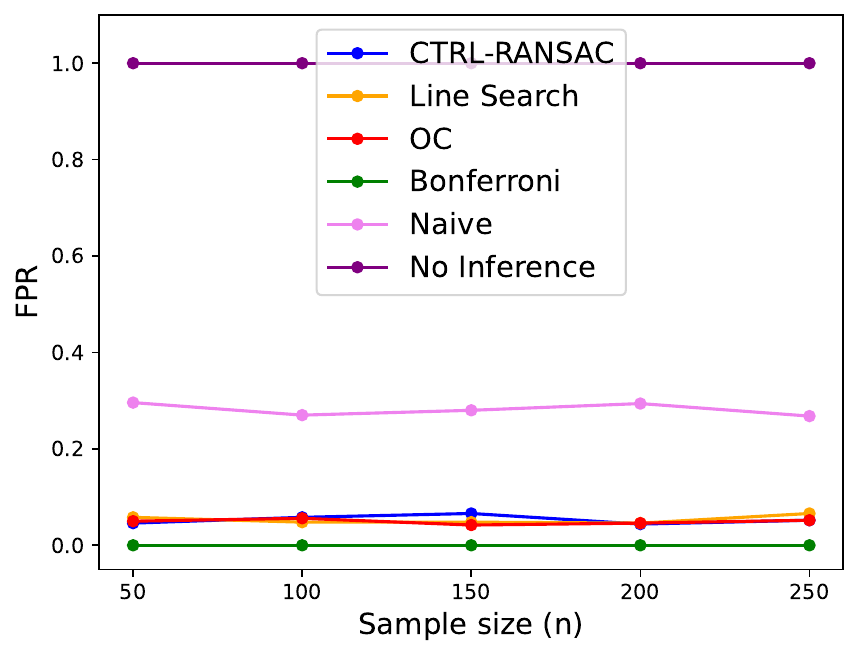}
        \caption{Correlation}
    \end{subfigure}
%        \vspace{-5pt}
    \caption{FPR Comparison}
    \label{fig:fpr_n_change}
    \vspace{-5pt}
\end{figure}

\begin{figure}[!t]
    \centering
    \begin{subfigure}[t]{0.24\textwidth}
        \centering
        \includegraphics[width=\textwidth]{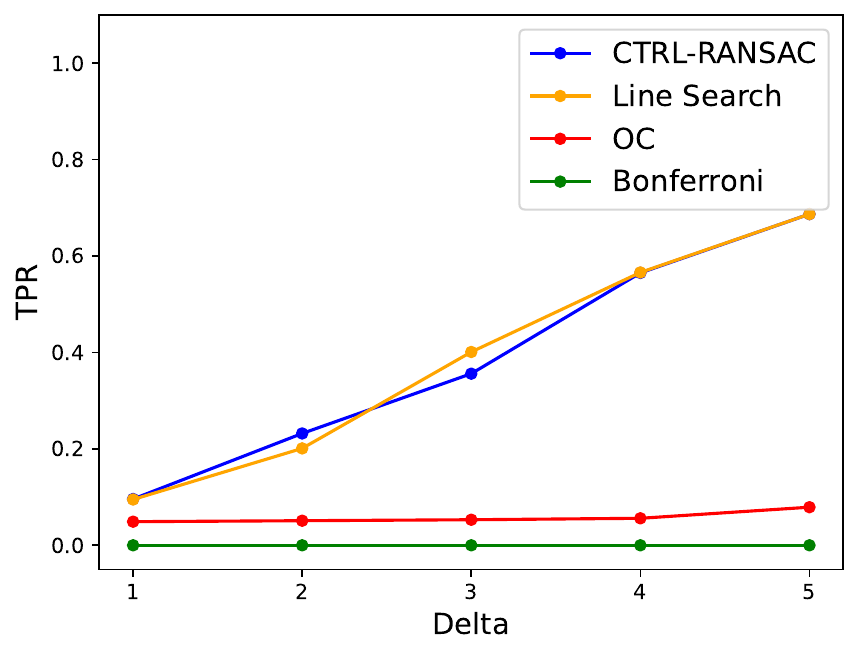}
        \caption{Independence}
    \end{subfigure}
    \hspace{-2mm}
    \begin{subfigure}[t]{0.24\textwidth}
        \centering
        \includegraphics[width=\textwidth]{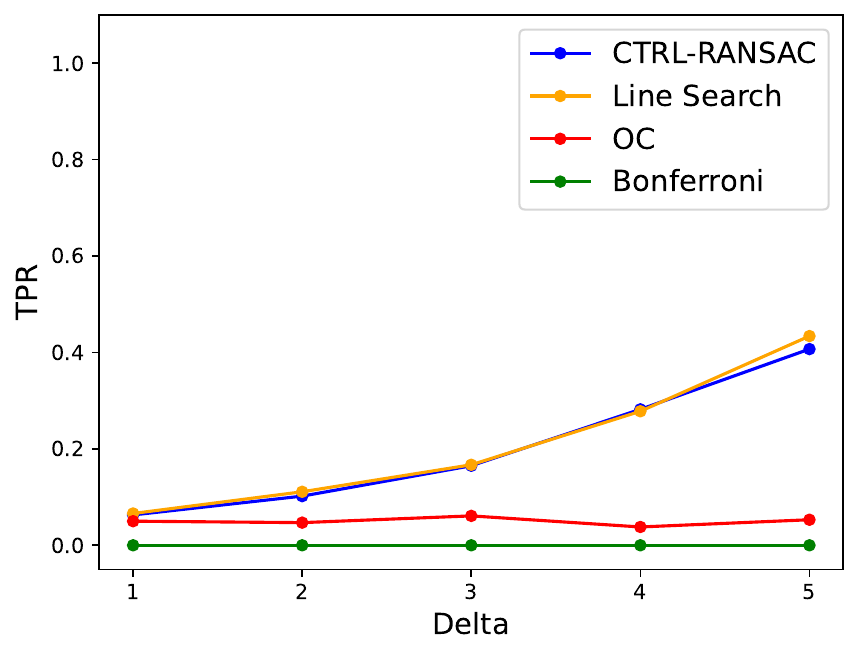}
        \caption{Correlation}
    \end{subfigure}
%        \vspace{-5pt}
    \caption{TPR Comparison}
    \label{fig:tpr_delta_change}
    \vspace{-5pt}
\end{figure}

\begin{figure}[!t]
    \centering
    \begin{subfigure}[t]{0.24\textwidth}
        \centering
        \includegraphics[width=\textwidth]{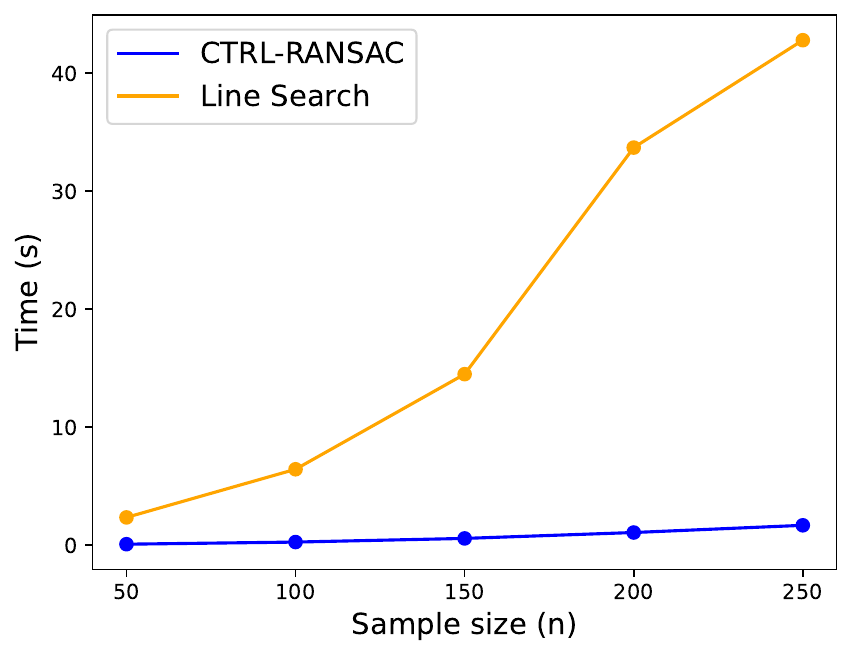}
        \caption{$n$ changes}
    \end{subfigure}
    \hspace{-2mm}
    \begin{subfigure}[t]{0.24\textwidth}
        \centering
        \includegraphics[width=\textwidth]{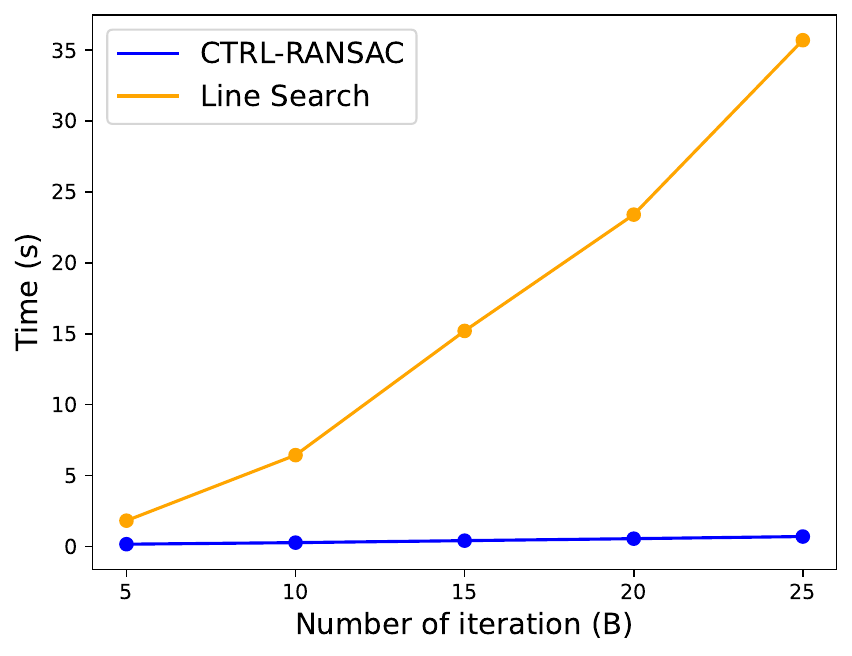}
        \caption{$B$ changes}
    \end{subfigure}
%        \vspace{-5pt}
    \caption{Average running time in 100 trials}
    \label{fig:computational_time}
    \vspace{-5pt}
\end{figure}

We generated the dataset $( X, \bm Y)$ using $\bm Y = X \bm \beta^\ast + \bm \veps$, where $X_{i} \sim \NN(\boldsymbol{0}^{p}, I^{p})$ with $\boldsymbol{0}^{p} \in \RR^p$ is a vector of zeros and $I^{p} \in \RR^{p \times p}$ is the identity matrix. The  $\bm \beta^\ast = (1, 2, 1, 2, ...)^\top \in \RR^p$ and $\bm \veps \sim N(\bm 0^n, \Sigma)$.
We considered two covariance matrices: $\Sigma = I^n$ (Independence) and 
$
\Sigma = \left [0.5^{|i - j|} \right ]_{ij}, \forall i,j \in [n]
$, (Correlation).
We set $p = 5$, $B = 15$ and $\tau = 2$.
For the FPR experiments, we set $n \in \{ 50, 100, 150, 200, 250 \}$.
For the TPR experiments, $\lfloor \frac{n}{5} \rfloor$ instances were randomly made to be anomalies by setting $\bm Y_{i} = \bm Y_{i} + \Delta$ where $\Delta \in \{ 1, 2, 3, 4, 5\}$. 
Each experiment was repeated 1000 times.
Additional experiments for various values of $B,  p$, and $\tau$ are provided in Appendix \ref{app:additional_experiment}.

\textbf{The results of FPRs and TPRs}.
The results of the FPR control are shown in Fig. \ref{fig:fpr_n_change}. 
The \texttt{CTRL-RANSAC}, {\tt Line Search, OC} and \texttt{Bonferroni} effectively controlled the FPR under $\alpha$, whereas the \texttt{Naive} and \texttt{No Inference} failed to do so.
Because \texttt{Naive} and \texttt{No Inference} could not control the FPR, we no longer considered their TPRs.
The TPR results are shown in Fig. \ref{fig:tpr_delta_change}.
The {\tt CTRL-RANSAC} and {\tt Line Search} has the highest TPR compared to other methods in all the cases, i.e., the {\tt CTRL-RANSAC} and {\tt Line Search} has the lowest FNR.
Although {\tt CTRL-RANSAC} and {\tt Line Search} exhibit high TPR, their computational costs differ significantly and will be discussed in the next paragraph.

\textbf{Computational time.} The results on computational time are shown in Fig. \ref{fig:computational_time}.
Our initial method, \texttt{Line Search}, developed for RANSAC-based AD, incurs high computational costs. This is because it employs the parametric line search approach from \citet{le2021parametric} and \citet{le2024cad}, which requires exhaustive searches across the data space. To address this, we later developed \texttt{CTRL-RANSAC}, which significantly reduces the computational burden. Further details on complexity analyses of these two methods are provided in Appendix \ref{app:complexity_analysis}.

\begin{table}[!t]
    \renewcommand{\arraystretch}{1.2}
    \centering
        \caption{Comparison on the \textit{Brownlee} dataset}.
    \label{tab:BrownleeDatasetResult}
    \begin{tabular}{|c|c|c|}
        \hline
         & ~ FPR ~ & ~ TPR ~ \\ \hline
         \texttt{CTRL-RANSAC} & \textbf{0.018} & \textbf{0.659} \\ \hline
         \texttt{Linear Search} & \textbf{0.018} & \textbf{0.659} \\ \hline
         \texttt{OC} & 0.04 & 0.041 \\ \hline
         \texttt{Naive} & 0.159 & \diagbox[width=1.6cm,height=0.55cm]{}{} \\ \hline
         \texttt{Bonferroni} & 0.0 & 0.365 \\ \hline
         \texttt{No inference} & 1.0 & \diagbox[width=1.6cm,height=0.55cm]{}{} \\ \hline
    \end{tabular}
        \vspace{-5pt}
\end{table}

\begin{table}[!t]
    \renewcommand{\arraystretch}{1.2}
    \centering
    \caption{Comparison on the \textit{Hill Races} dataset}.
    \label{tab:HillDatasetResult}
    \begin{tabular}{|c|c|c|}
        \hline
         & ~ FPR ~ & ~ TPR ~ \\ \hline
         \texttt{CTRL-RANSAC} & \textbf{0.027} & \textbf{0.775} \\ \hline
         \texttt{Line Search} & \textbf{0.027} & \textbf{0.775} \\ \hline
         \texttt{OC} & 0.048 & 0.091 \\ \hline
         \texttt{Naive} & 0.299 & \diagbox[width=1.6cm,height=0.55cm]{}{} \\ \hline
         \texttt{Bonferroni} & 0.0 & 0.658 \\ \hline
         \texttt{No inference} & 1.0 & \diagbox[width=1.6cm,height=0.55cm]{}{} \\ \hline
    \end{tabular}
        \vspace{-5pt}
\end{table}

\subsection{Real-data Experiments}

We performed the comparison using three real-world datasets: the Brownlee’s Stack Loss Plant Data (referred to as Brownlee), the Hill Races Data, and the Motor Trend Car Road Tests (referred to as Mtcars).
The first two datasets are previously used in the context of AD \citep{andrews1974robust, hoeting1996method}.
The last one is widely used in {\tt R}.
In the first two datasets, it is known that certain instances are considered outliers, making them suitable for evaluating FPR and TPR.
The experiment was repeated 1000 times.
The results are shown in Tabs. \ref{tab:BrownleeDatasetResult} and \ref{tab:HillDatasetResult}.
The  {\tt CTRL-RANSAC} and {\tt Line Search} had the highest TPR while properly controlled the FPR under $\alpha = 0.05$.
We note that the computational time of the {\tt Line Search} is significantly higher than that of  the {\tt CTRL-RANSAC}.
Additionally, we compared the $p$-values of the \texttt{CTRL-RANSAC} and \texttt{OC} on the three datasets. We excluded \texttt{Line Search} because it exhibited the same FPR and TPR as \texttt{CTRL-RANSAC}.
The boxplots of the distribution of the $p$-values are illustrated in Figs. \ref{fig:boxplot_p_value_1} and \ref{fig:boxplot_p_value_2}.
The $p$-values of the \texttt{CTRL-RANSAC} tend to be smaller than those of \texttt{OC}, which indicates that the \texttt{CTRL-RANSAC} method has higher power than the \texttt{OC}.

\begin{figure}[!t]
    \centering
    \begin{subfigure}[t]{0.24\textwidth}
        \centering
        \includegraphics[width=\textwidth]{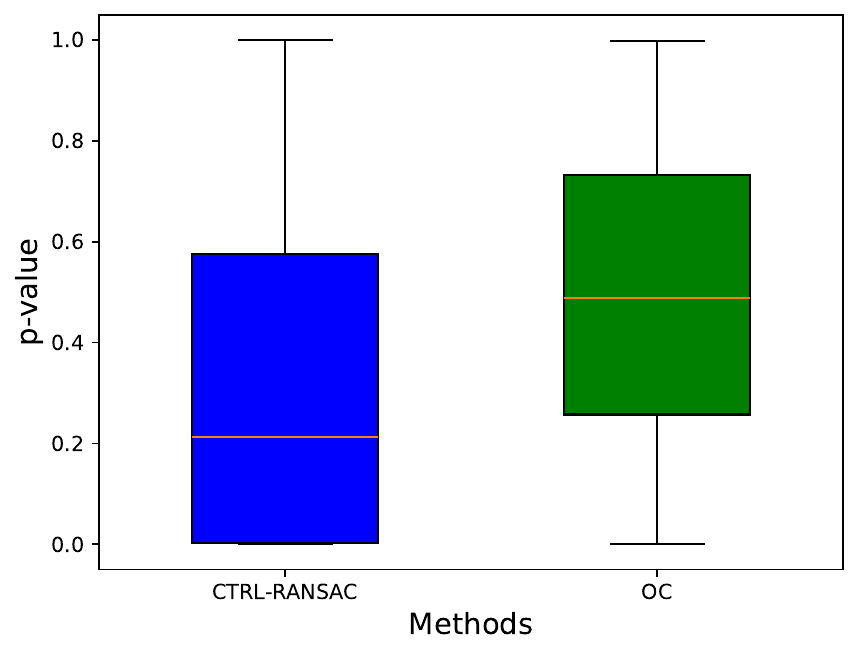}
        \caption{Brownlee dataset}
    \end{subfigure}
    \hspace{-2mm}
    \begin{subfigure}[t]{0.24\textwidth}
        \centering
        \includegraphics[width=\textwidth]{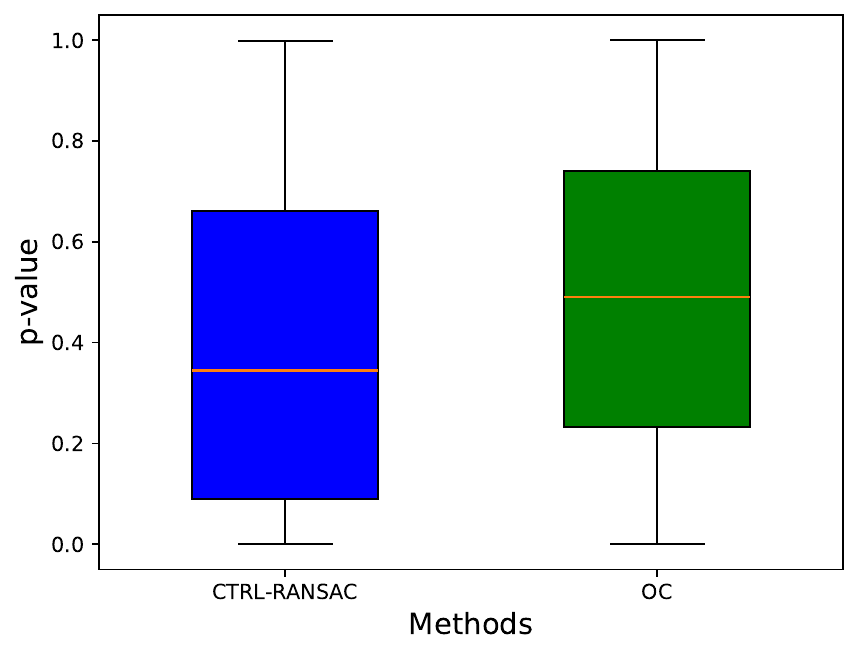}
        \caption{Hill Races dataset}
    \end{subfigure}
%        \vspace{-8pt}
    \caption{Boxplots of $p$-value}
    \label{fig:boxplot_p_value_1}
        \vspace{-2pt}
\end{figure}

\begin{figure}[!t]
    \centering
    \includegraphics[width=.62\linewidth]{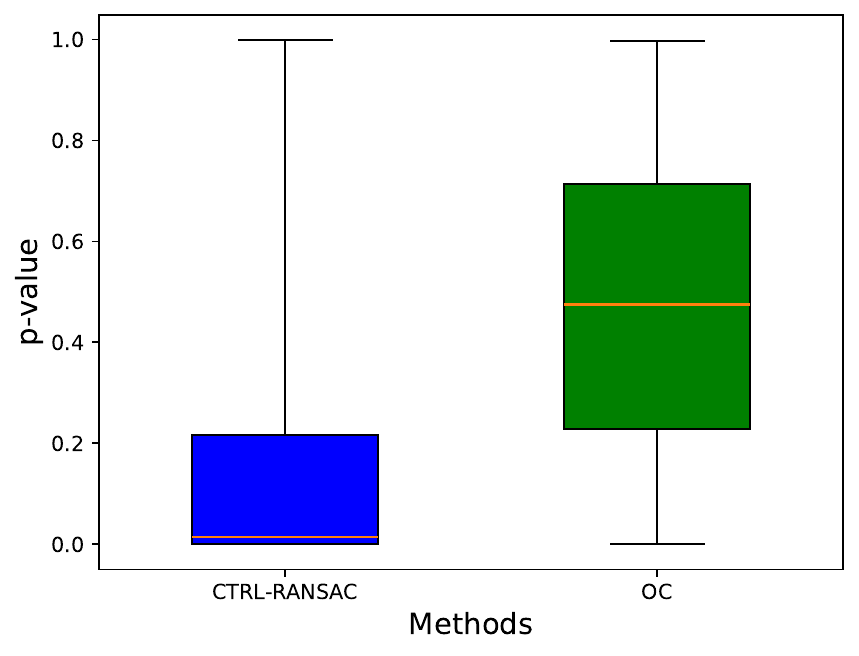}
        \vspace{-5pt}
    \caption{Boxplots of $p$-value on the Mtcars dataset}
    \label{fig:boxplot_p_value_2}
        \vspace{-5pt}
\end{figure}

%% file: sec5.tex
\vspace{-5pt}
\section{Discussion} \label{sec:discussion}

\vspace{-5pt}
We propose a novel method to control the FPR under a given significance level $\alpha$ while maintaining a high TPR for statistical testing on anomalies detected by RANSAC. Our approach leverages the SI framework, incorporating the ``divide-and-conquer'' strategy and dynamic programming to efficiently compute the \textit{p}-value. This study represents a significant advancement toward reliable artificial intelligence. The current method does not account for selection events when calculating the threshold $\tau$ or estimating the variance of the noise distribution. Extending the proposed approach to handle these aspects within the SI framework would be a valuable future contribution.

%% file: appendix.tex
\section{Appendix}
\label{sec:appendix}

\subsection{Proof of Lemma \ref{lemma:valid_selective_p}} \label{appendix:valid_p_proof}

%We are considering the distribution
%\begin{equation}
%\label{eq:si_conditioning_distribution}
%    \bm \eta_i^\top \bm Y
%    ~ \Big | 
%    \left \{ \cO(\bm Y) = \cO^{\rm obs},
%    \cQ(\bm Y) = \cQ^{\rm obs} \right \}.
%\end{equation}
%Since $\cQ(\bm Y)$ is independent with $\bm \eta_i^\top \bm Y$, we can eliminate it from the condition. Then \eq{eq:si_conditioning_distribution} can be rewritten as
%\begin{equation}
%    \bm \eta_i^\top \bm Y 
%    ~ \Big | 
%    \left \{ \cO(\bm Y) = \cO^{\rm obs} \right \}.
%\end{equation}

%Furthermore, we have 
We have 
\begin{equation}
    \bm \eta_i^\top \bm Y 
    ~ \Big | 
    \left \{ \cO(\bm Y) = \cO^{\rm obs}, \cQ(\bm Y) = \cQ^{\rm obs} \right \} \sim {\rm TN} \big( \boldsymbol{\eta}_{i}^{\top}\boldsymbol{\mu}, \boldsymbol{\eta}^{\top}_{i}\Sigma\boldsymbol{\eta}_{i}, \mathcal{Z} \big),
\end{equation} 
where $\cZ$ is the truncation region described in \S \ref{sec:method}.
Thus, under the null hypothesis, the selective \textit{p}-value defined in \eq{eq:selective_p} follows the uniform distribution $\rm Unif(0, 1)$. Therefore,
\begin{equation}
    \bP_{{\rm H}_{0, i}} 
    \Big ( 
	p_i^{\rm selective} \leq \alpha
	~
	\Big | 
	~
        \cO(\bm Y) = \cO^{\rm obs},
        \cQ(\bm Y) = \cQ^{\rm obs} 
    \Big ) 
    = 
    \alpha,
    \qquad
    \forall \alpha \in [0, 1].
\end{equation}
By integrating over $\cQ (\bm Y)$, we have
\begin{align*}
    & \bP_{{\rm H}_{0, i}} 
    \Big ( 
	p_i^{\rm selective} \leq \alpha
	~
	\Big | 
	~
        \cO(\bm Y) = \cO^{\rm obs} 
    \Big ) \\
    & = \int
    \bP_{{\rm H}_{0, i}} 
    \Big ( 
	p_i^{\rm selective} \leq \alpha
	~
	\Big | 
	~
        \cO(\bm Y) = \cO^{\rm obs},
        \cQ(\bm Y) = \cQ^{\rm obs} 
    \Big )
    \bP_{{\rm H}_{0, i}} 
    \Big ( 
	\cQ(\bm Y) = \cQ^{\rm obs}
	~
	\Big | 
	~
        \cO(\bm Y) = \cO^{\rm obs} 
    \Big ) d\cQ^{\rm obs} \\
    & = \int 
    \alpha 
    \bP_{{\rm H}_{0, i}} 
    \Big ( 
	\cQ(\bm Y) = \cQ^{\rm obs}
	~
	\Big | 
	~
        \cO(\bm Y) = \cO^{\rm obs} 
    \Big ) d\cQ^{\rm obs} \\
    & = \alpha
    \int 
    \bP_{{\rm H}_{0, i}} 
    \Big ( 
	\cQ(\bm Y) = \cQ^{\rm obs}
	~
	\Big | 
	~
        \cO(\bm Y) = \cO^{\rm obs} 
    \Big ) d\cQ^{\rm obs} \\
    & = \alpha.
\end{align*}
Finally, we obtain Lemma 1 as follows:
\begin{align*}
    & \bP_{{\rm H}_{0, i}} 
    \Big ( 
	p_i^{\rm selective} \leq \alpha
    \Big ) \\
    & = \sum_{\cO^{\rm obs}}
    \bP_{{\rm H}_{0, i}} 
    \Big ( 
	p_i^{\rm selective} \leq \alpha
	~
	\Big | 
	~
        \cO(\bm Y) = \cO^{\rm obs}
    \Big )
    \bP_{{\rm H}_{0, i}} 
    \Big ( 
	\cO(\bm Y) = \cO^{\rm obs} 
    \Big ) \\
    & = \sum_{\cO^{\rm obs}} 
    \alpha 
    \bP_{{\rm H}_{0, i}} 
    \Big ( 
	\cO(\bm Y) = \cO^{\rm obs} 
    \Big ) \\
    %d\cO^{\rm obs} \\
    & = \alpha
    \sum_{\cO^{\rm obs}} 
    \bP_{{\rm H}_{0, i}} 
    \Big ( 
        \cO(\bm Y) = \cO^{\rm obs} 
    \Big )  \\
    & = \alpha.
\end{align*}

%\subsection{Proof of independence between $\mathcal{Q}(\boldsymbol{y})$ and $Z$} \label{ProofOfIndependence}
%\begin{align*}
%    cov\big(Z, (I_{n}-\boldsymbol{b}\boldsymbol{\eta}^{\top}_{j})\boldsymbol{y} \big) 
%    & = cov \big(\boldsymbol{\eta}^{\top}_{j}\boldsymbol{y}, (I_{n}-\boldsymbol{b}\boldsymbol{\eta}^{\top}_{j})\boldsymbol{y} \big) \\
%    & = \boldsymbol{\eta}^{\top}_{j} cov\big( \boldsymbol{y}, \boldsymbol{y}\big) (I_{n}-\boldsymbol{b}\boldsymbol{\eta}^{\top}_{j})^{\top} \\ 
%    & = \boldsymbol{\eta}^{\top}_{j} \Sigma (I_{n} - \boldsymbol{\eta}_{j} \boldsymbol{b}^{\top}) \\
%    & = \boldsymbol{\eta}^{\top}_{j}\Sigma-\boldsymbol{\eta}^{\top}_{j}\Sigma \boldsymbol{\eta}_{j} \frac{(\Sigma \boldsymbol{\eta}_{j})^{\top}}{\boldsymbol{\eta}^{\top}_{j}\Sigma \boldsymbol{\eta}_{j}} \\ 
%    & = \boldsymbol{\eta}^{\top}_{j}\Sigma - \boldsymbol{\eta}^{\top}_{j}\Sigma = 0.
%\end{align*}
%Since the covariance between $Z$ and $\mathcal{Q}(\boldsymbol{y})$ is 0, it follows that they are independent.

\subsection{Proof of Lemma \ref{lemma:data_line}} \label{appendix:data_line_proof}

According to the second condition in (\ref{eq:conditional_data_space}), we have 
\begin{align*}
    \cQ (\bm Y) &= \cQ ^{\rm obs} \\
    \Leftrightarrow (I^n - \bm b \bm \eta_i^\top ) \bm Y &= \cQ^{\rm obs} \\
    \Leftrightarrow \bm Y &= \cQ^{\rm obs} + \bm b \bm \eta_i^\top \bm Y.
\end{align*}
By defining $\bm a = \cQ ^{\rm obs}$, $z = \bm \eta_i^\top \bm Y$, and incorporating the first condition of (\ref{eq:conditional_data_space}), we obtain Lemma \ref{lemma:data_line}.

% We decompose $\boldsymbol{y}$ as:
% \begin{equation}
%     \boldsymbol{y} = 
%     \Big( 
%         I_{n}
%         - \frac{\Sigma \boldsymbol{\eta}_{j}}{\boldsymbol{\eta}^{\top}_{j}\Sigma\boldsymbol{\eta}_{j}} \boldsymbol{\eta}^{\top}_{j} 
%     \Big)
%     \boldsymbol{y}
%     + \frac{\Sigma \boldsymbol{\eta}_{j}}{\boldsymbol{\eta}^{\top}_{j}\Sigma\boldsymbol{\eta}_{j}} \boldsymbol{\eta}^{\top}_{j} \boldsymbol{y},
%     \forall \boldsymbol{y} \in \mathbb{R}^{n}.
% \end{equation}
% Then, by defining $\mathcal{Q}(\boldsymbol{y})$, $\boldsymbol{b}$ as in (\ref{nuisanceComponentFomula}) and $Z := \boldsymbol{\eta}^{\top}_{j} \boldsymbol{y}$,
% we have:
% \begin{equation} \label{decomposedY}
%     \boldsymbol{y} = \mathcal{Q}(\boldsymbol{y}) + \boldsymbol{b}Z, \forall \boldsymbol{y} \in \mathbb{R}^{n}.
% \end{equation}
% After that, we incorporate (\ref{decomposedY}) with the second condition in (\ref{DataSpaceY}):
% \begin{equation}
%     \boldsymbol{y} = \mathcal{Q}^{obs} + \boldsymbol{b}Z, \forall \boldsymbol{y} \in \mathcal{Y}.
% \end{equation}
% We then define $\boldsymbol{a} := \mathcal{Q}^{obs}$ and obtain Lemma \ref{conditionalDataSpaceLemma}.

\subsection{Proof of Lemma \ref{lemma:cZ_b_2}} \label{app:proof_cZ_b_2}

The $\cR^{(b)}_i$ in \eq{eq:cR_b_i} denotes the region of $z$ where $\left( \bm Y_i (z) -X_i^\top \hat{\bm \beta}^{(b)} (z) \right)^2 \leq \tau$. In other words, it represents the region where the deviation of $i^{\rm th}$ data point satisfies the criteria to be considered as an inlier when applying RANSAC with $\bm Y (z) = \bm a + \bm b z$. We have
\begin{align*}
    \left( \bm Y_i (z) -X_i^\top \hat{\bm \beta}^{(b)} (z) \right)^2 & \leq \tau \\
    \Leftrightarrow 
    \; \bm Y_i^2 (z) - 2 \bm Y_i (z) X_i^\top \hat{\bm \beta}^{(b)} (z) + \left( X_i^\top \hat{\bm \beta}^{(b)} (z) \right)^2 & \leq \tau . 
\end{align*}

For notational simplicity, we define $\bm \omega = \big( X_i^\top ( X_{s^{(b)}})^{+} I^n_{s^{(b)}} \big)^\top$. Then, we have 
\begin{align*}
    (\bm Y (z) )^\top \bm e_i \bm e_i^\top \bm Y(z) - 2 (\bm Y (z) )^\top \bm e_i \bm \omega^\top \bm Y (z) + (\bm Y (z) )^\top \bm \omega \bm \omega^\top \bm Y (z) &\leq \tau \\
    \Leftrightarrow \;
    (\bm Y (z) )^\top (\bm e_i \bm e_i^\top - 2 \bm e_i \bm \omega^\top + \bm \omega \bm \omega^\top) \bm Y (z) &\leq \tau \\
    \Leftrightarrow \; 
    (\bm a + \bm b z)^\top (\bm e_i \bm e_i^\top - 2 \bm e_i \bm \omega^\top + \bm \omega \bm \omega^\top) (\bm a + \bm b z) &\leq \tau.
\end{align*}
For notational simplicity, we define $\cH = \bm e_i \bm e_i^\top - 2 \bm e_i \bm \omega^\top + \bm \omega \bm \omega^\top$. Then we have
\begin{align*}
    \bm a ^\top \cH \bm a + a^\top \cH \bm b z + z^\top \bm b^\top \cH \bm a + z^\top \bm b ^\top \cH \bm b z &\leq \tau \\
    \Leftrightarrow \;
    \bm a ^\top \cH \bm a - \tau + \bm a ^\top (\cH +\cH ^\top)\bm b z + \bm b ^\top \cH \bm b z^2 &\leq 0.
\end{align*}
Therefore, $\cR _i^{(b)}$ can be computed in term of a quadratic inequality w.r.t. $z$ as:
\begin{align*}
	\cR^{(b)}_i = 
	\Big \{ 
		z \in \RR \mid 
		      \bm a ^\top \cH \bm a 
                - \tau 
                + \bm a ^\top (\cH +\cH ^\top)\bm b z 
                + \bm b ^\top \cH \bm b z^2 
                \leq 0
	\Big \}.
\end{align*}
Finally, we obtain Lemma \ref{lemma:cZ_b_2}.

\subsection{Proof of Lemma \ref{lemma:cZ_b_1}} \label{app:proof_cZ_b_1}

%The set $\cZ^{(b)}_1$ can be re-written as follows:
%%
%\begin{align*}
%    \cZ^{(b)}_1
%    = 
%    \Bigg( \bigcap_{u \in [B], u\neq b} \cS^{(u)}_{n, |\cO^{\rm obs}|}\Bigg) \bigcap \Bigg( \cS^{(b)}_{n,|\cO^{\rm obs}| - 1} \setminus\cS^{(b)}_{n, |\cO^{\rm obs}|} \Bigg),
%\end{align*}
%%
%where $\cS^{(u)}_{n, |\cO^{\rm obs}|}$ is defined as:
%%
%\begin{align} \label{eq:cS_ub_n}
%    \cS^{(u)}_{n, |\cO^{\rm obs}|} = 
%    \left \{ 
%    	z \in \RR :
%    	\big | \cO^{(u)}_n (z)
%    	\big |
%    	>
%    	\big | \cO^{\rm obs} \big|
%    \right \},
%\end{align}
%%
%$\cO^{(u)}_n (z)$ is the set of outliers whose indices are smaller than or equal to $n$, detected by the $u^{\rm th}$ model when applying RANSAC to the vector $\bm Y (z) = \bm a + \bm b z$.

The $\cZ_1^{(b)}$ denotes the region of $z$ where the $b^{\rm th}$ model is optimal and $\big| \cO^{(b)} (z) \big| = \big| \cO^{\rm obs} \big|$ (we note that $\cO^{(b)} (z) = \cO^{(b)}_n (z)$). In other words, this region satisfies two conditions:

$\bullet$ \textbf{ Condition 1}: The outlier set of the $b^{\rm th}$ model has the same cardinality as the observed set of outliers, i.e., $\big| \cO^{(b)}_n (z) \big| = \big| \cO^{\rm obs} \big|$.

$\bullet$ \textbf{ Condition 2}: For all $u \in [B]$ where $u \neq b$, the outlier set for the $u^{\rm th}$ model has a cardinality strictly greater than $\big| \cO^{\rm obs} \big|$. That is, $\forall u \in [B], u \neq b: \big| \cO^{(u)}_n (z) \big| > \big| \cO^{\rm obs} \big|$. This implies that the $b^{\rm th}$ model is optimal, as it has the smallest set of outliers compared to other models.

The first condition can be written %in terms of $\cS$ 
as follows: 
\begin{align*}
    \cS^{(b)}_{n,|\cO^{\rm obs}| - 1} \setminus\cS^{(b)}_{n, |\cO^{\rm obs}|} 
    & = \left\{ z \in \RR : \big | \cO^{(b)}_n (z) \big | > \big | \cO^{\rm obs} \big| - 1 \right\} \setminus \left\{ z \in \RR : \big | \cO^{(b)}_n (z) \big | > \big | \cO^{\rm obs} \big| \right\} 
    \\ 
    & = \left\{ z \in \RR : \big| \cO^{\rm obs} \big| \geq \big | \cO^{(b)}_n (z) \big | > \big | \cO^{\rm obs} \big| - 1 \right\} \\
    & = \left\{ z \in \RR : \big | \cO^{(b)}_n (z) \big | = \big | \cO^{\rm obs} \big| \right\}.
\end{align*}

Similarly, the second condition can be expressed %in terms of $\cS$ 
as:
\begin{align*}
    \bigcap_{u \in [B], u\neq b} \cS^{(u)}_{n, |\cO^{\rm obs}|} 
    & = \bigcap_{u \in [B], u\neq b} \left\{ z \in \RR : \big | \cO^{(u)}_n (z) \big | > \big | \cO^{\rm obs} \big| \right\}
    \\ 
    & = \left\{ z \in \RR :  \big | \cO^{(u)}_n (z) \big | > \big | \cO^{\rm obs} \big|, \forall u \in [B], u \neq b \right\}.
\end{align*}

Finally, the region $\cZ^{(b)}_1$ is obtained by intersecting two regions corresponding to these conditions: \begin{equation*}
    \cZ^{(b)}_1 = \Bigg( \bigcap_{u \in [B], u\neq b} \cS^{(u)}_{n, |\cO^{\rm obs}|}\Bigg) \bigcap \Bigg( \cS^{(b)}_{n,|\cO^{\rm obs}| - 1} \setminus\cS^{(b)}_{n, |\cO^{\rm obs}|} \Bigg).
\end{equation*}

\subsection{Proof of Lemma \ref{lemma:recursive}} \label{app:proof_recursive}

The $\cS^{(u)}_{j, k}$, as defined in Lemma \ref{lemma:recursive}, represents the region of $z$ where, when applying RANSAC with $\bm Y (z) = \bm a + \bm b z$, the $u^{\rm th}$ model identifies MORE THAN $k$ outliers among the first $j$ data points. The recursive formula in Lemma \ref{lemma:recursive} defines $\cS^{(u)}_{j, k}$ under three distinct cases. It is important to note that the recurrence relation depends on two parameters: $j$ and $k$.

$\bullet \textbf{ Case } j \leq k$: Since it is impossible to detect more than $k$ outliers from the first $j$ data points when $j \leq k$, $\cS^{(u)}_{j, k} = \emptyset$. For instance, identifying two or more outliers is impossible if only one data point has been observed.

$\bullet \textbf{ Case } k = 0$: $\cS^{(u)}_{j, 0}$ represents the region of $z$ where at least one outlier is detected among the first $j$ data points. This can be expressed as the union of the regions for all individual data points: $\cS^{(u)}_{j, 0} = \cup_{l = 1}^{j} \overline{\cR}^{(u)}_l$. To apply the recursive formula, we express this as $\cS^{(u)}_{j, 0} = \overline{\cR}^{(u)}_j \cup \cS^{(u)}_{j - 1, 0}$, where $j \neq 0$ (as $j = 0$ is handled in the first case).

$\bullet$ \textbf{General case}: Let $G$ be the set of all possible outlier sets, each of length larger than $k$, when classifying the first $j$ data points by the $u^{\rm th}$ model. We divide it into two subsets:
\begin{equation*}
        G^- := \{ g \in G \mid j \notin g \}, \;\;
        G^+ := \{ g \in G \mid j \in g \}.
\end{equation*}
Here, $G^-$ is the set of all possible outlier sets that exclude $j^{\rm th}$ data point, indicating that this data point is an inlier in the $u^{\rm th}$ model. Conversely, $G^+$ is the set of all possible outlier sets that include $j^{\rm th}$ data point, indicating that this data point is an outlier in $i^{\rm th}$ model.

For $g \in G^-$, the $j^{\rm th}$ data point is an inlier, meaning that more than $k$ outliers must be within the first $j - 1$ data points. This leads to the condition $\cR^{(u)}_j \cap \cS^{(u)}_{j - 1, k} = \left \{ z \in \RR : \big( j \notin \cO^{(u)}_j (z) \big ) \wedge \big ( \big | \cO^{(u)}_{j-1} (z) \big | > k \big ) \right \}$.

For $g \in G^+$, the $j^{\rm th}$ data point is an outlier, requiring at least $k$ other outliers among the first $j - 1$ data points. This leads to the condition $\overline{\cR}^{(u)}_j \cap \cS^{(u)}_{j - 1, k - 1} = \left \{ z \in \RR : \big ( j \in \cO^{(u)}_j (z) \big ) \wedge \big ( \big | \cO^{(u)}_{j-1} (z) \big | > k-1 \big ) \right \}$.

By taking the union of these two regions, we obtain:
\begin{equation*}
    \cS^{(u)}_{j, k} = \left \{ z \in \RR : \big | \cO^{(u)}_j (z) \big | > k \right \} = \Big( \cR^{(u)}_j \cap \cS^{(u)}_{j - 1, k} \Big) \cup \Big( \overline{\cR}^{(u)}_j \cap \cS^{(u)}_{j - 1, k - 1} \Big).
\end{equation*}

\subsection{The extension of Lemma \ref{lemma:cZ_b_1} to the case where multiple models share the smallest outlier set} \label{app:extension_lemma_cZ_1}
As mentioned before, when multiple models have the same largest set of inliers i.e. smallest set of outliers, the first encountered model is selected. In other words, if the $b^{\rm th}$ model is deemed optimal, no model encountered before it has a smaller or equal set of outliers. However, models encountered after it may have the same minimal set of outliers. Therefore, $\cZ_1^{(b)}$ can be written as:
\begin{align*}
	\cZ^{(b)}_1
	= 
	\Bigg( \bigcap^{b-1}_{u=1} \cS^{(u)}_{n, |\cO^{\rm obs}|} \Bigg) 
        \bigcap \Bigg( \bigcap^{B}_{u=b+1} \cS^{(u)}_{n, |\cO^{\rm obs}|-1} \Bigg)
        \bigcap \Bigg(\cS^{(b)}_{n, |\cO^{\rm obs}| - 1} \setminus \cS^{(b)}_{n, |\cO^{\rm obs}|}\Bigg).
\end{align*}

%Theoretically, considering this case could result in a higher true positive rate (TPR) thanks to the larger truncation region $\cZ$. However, the difference is relatively small and we can hardly notice it through experiments.

\subsection{Analysis of the Complexity of Algorithm \ref{alg:ctrl_ransac}} \label{app:complexity_analysis}

% In Lemma \ref{lemma:cZ_b_2}, computing $\cZ^{(b)}_2$ has a time complexity of $O(n)$, assuming that the computation of each $\cR^{(b)}_i$ is constant. On the other hand, $\cZ^{(b)}_1$ is more computationally demanding. According to Lemma \ref{lemma:cZ_b_1}, $\cZ^{(b)}_1$ is defined as the intersection of $\cS^{(u)}_{n, |\cO^{\rm obs}|}, \forall u \in [B], u \neq b$. Since $\cS^{(u)}_{n, |\cO^{\rm obs}|}$ is computed recursively, as per Lemma \ref{lemma:recursive}, this requires $O(n^2)$ time. Therefore, computing $\cZ^{(b)}_1$ takes $O(B \times n^2)$ time. In Algorithm \ref{alg:identify_truncation_region}, the overall complexity is $O(B(n + B \times n^2))$ since, for each $b \in [B]$, both $\cZ^{(b)}_1 \text{ and } \cZ^{(b)}_2$ must be determined. By the dominance principle in Big O notation, the complexity simplifies to $O(B^2 \times n^2)$. In Algorithm \ref{alg:ctrl_ransac}, the dominant term is the computation of $\cZ$ which, as shown earlier, has a complexity of $O(B^2 \times n^2)$. Thus, the overall complexity of Algorithm \ref{alg:ctrl_ransac} is $O(|\cO^{\rm obs}| \times B^2 \times n^2)$.

\textbf{Complexity.} In Lemma \ref{lemma:cZ_b_2}, computing each $\cZ^{(b)}_2$ has a time complexity of $O(n)$, assuming that the computation of each $\cR^{(b)}_i$ is constant. Therefore, computing all $\cZ^{(b)}_2$ requires $O(B\times n)$ time. On the other hand, $\cZ^{(b)}_1$ is more computationally demanding. According to Lemma \ref{lemma:cZ_b_1}, $\cZ^{(b)}_1$ can only be computed once $\cS_{n, |\cO^{\rm obs}|}^{(u)}$ has been calculated for all $u \in [B]$. Since $\cS^{(u)}_{n, |\cO^{\rm obs}|}$ is computed recursively, as per Lemma \ref{lemma:recursive}, this requires $O(n^2)$ time. Therefore, computing each $\cZ^{(b)}_1$ takes $O(B \times n^2)$ time. Consequently, computing $\cZ^{(b)}_1$ for all $b\in[B]$ requires $O(B^2 \times n^2)$ time, making it the bottleneck of the Algorithm \ref{alg:identify_truncation_region}. In algorithm \ref{alg:ctrl_ransac}, the most time-consuming step turns out to be the identification of $\cZ$ which, as shown earlier, has a complexity of $O(B^2 \times n^2)$. Thus, the overall complexity of CTRL-RANSAC applying to each outlier is $O(B^2\times n^2)$.

\textbf{Computational trick.} In fact, we can calculate $\cR_i^{(u)}, \forall u \in [B], i\in [n]$ and $\cS_{j, k}^{(u)}, \forall u\in[B], j,k\in[n]$ in advance. This implies that they can be computed before the for loop in Algorithm \ref{alg:identify_truncation_region} and need to be computed only once. Employing this trick, we can reduce the complexity of computing each $\cZ_1^{b}$ to $O(n)$ instead of $O(B\times n^2)$. But then, the bottleneck of the Algorithm \ref{alg:identify_truncation_region} becomes the computation of $S_{n, |\cO^{\rm obs}|}^{(u)}$ for all $u\in[B]$, which has the time complexity of $O(B\times n^2)$. Thus, the overall complexity of CTRL-RANSAC applying to each outlier is reduced to $O(B\times n^2)$.

\subsection{Additional Experiments} \label{app:additional_experiment}
Here, we provide additional numerical experiments including more parameter changes and the robustness of CTRL-RANSAC.

\textbf{Parameter changes.} In these experiments, we set the default value of $n = 150$. For the TPR experiments, we set $\Delta = 3$ and $n\in \{ 50,100,150,200,250 \} $. Other parameters were set as follows:
\begin{itemize}
    \item $B \in \{ 5,10,15,20,25 \}$,
    \item $p \in \{ 1,3,5,7,9 \}$ and
    \item $\tau \in \{ 1, 1.5, 2, 2.5, 3 \}$.
\end{itemize}

The results of the FPRs and TPRs are shown in Fig. \ref{fig:fpr_additional} and \ref{fig:tpr_additional} respectively. The CTRL-RANSAC, Line Search, OC and Bonferroni still effectively control the FPR under $\alpha$. The CTRL-RANSAC and Line Search still have the same highest performance on TPR.

\begin{figure*}[!t]
    \centering
    \begin{minipage}{\textwidth}
        \centering
        
        % First row with a row caption
        \subfloat[]{\includegraphics[width=0.33\textwidth]{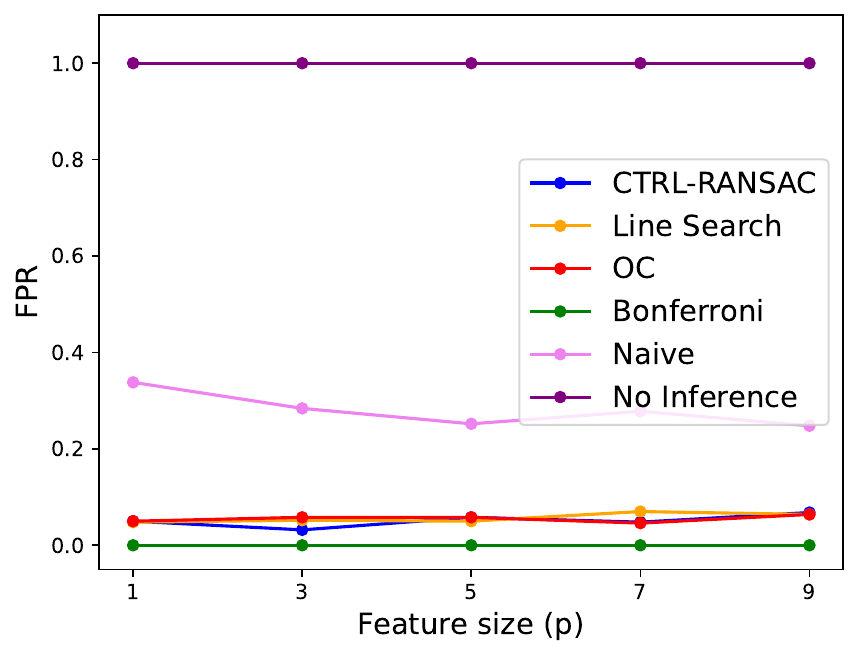}}\hfill
        \subfloat[]{\includegraphics[width=0.33\textwidth]{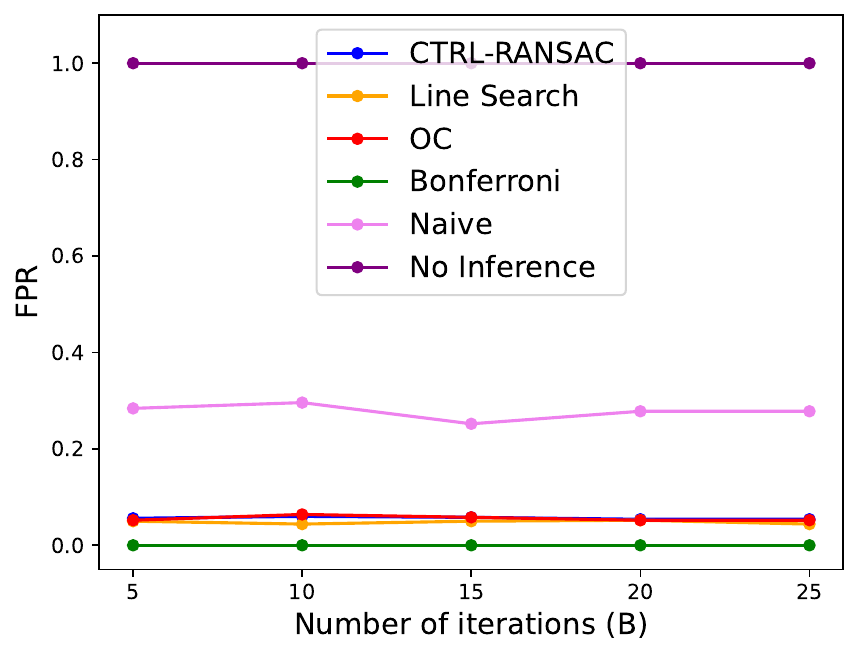}}\hfill
        \subfloat[]{\includegraphics[width=0.33\textwidth]{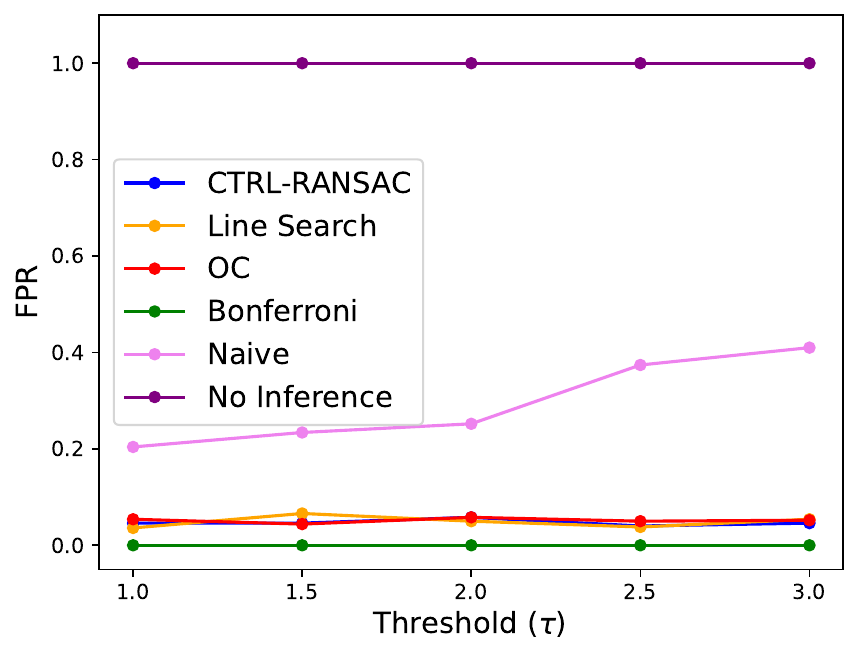}}\\
        (1) Independence \\[0.5em]
        
        % Second row with a row caption
        \subfloat[]{\includegraphics[width=0.33\textwidth]{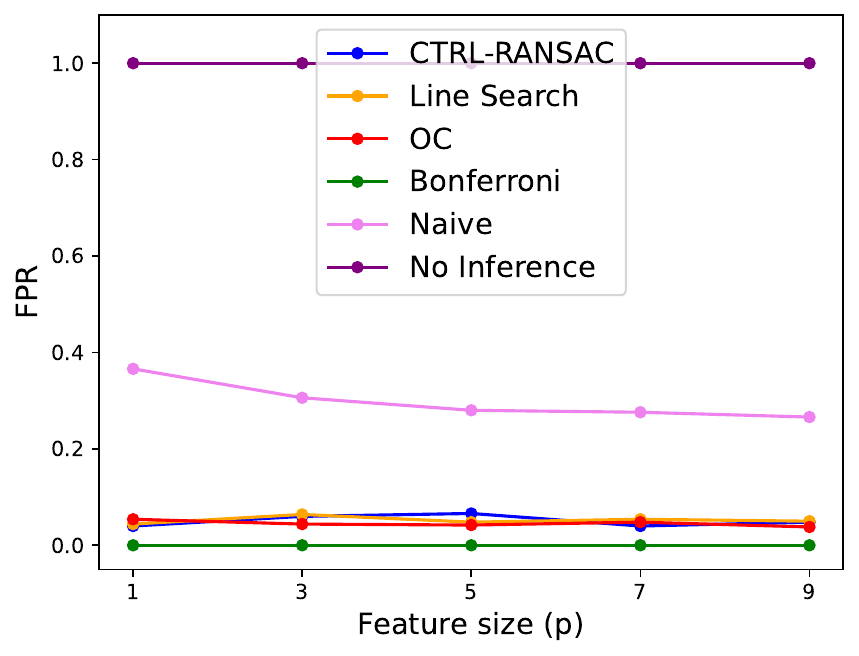}}\hfill
        \subfloat[]{\includegraphics[width=0.33\textwidth]{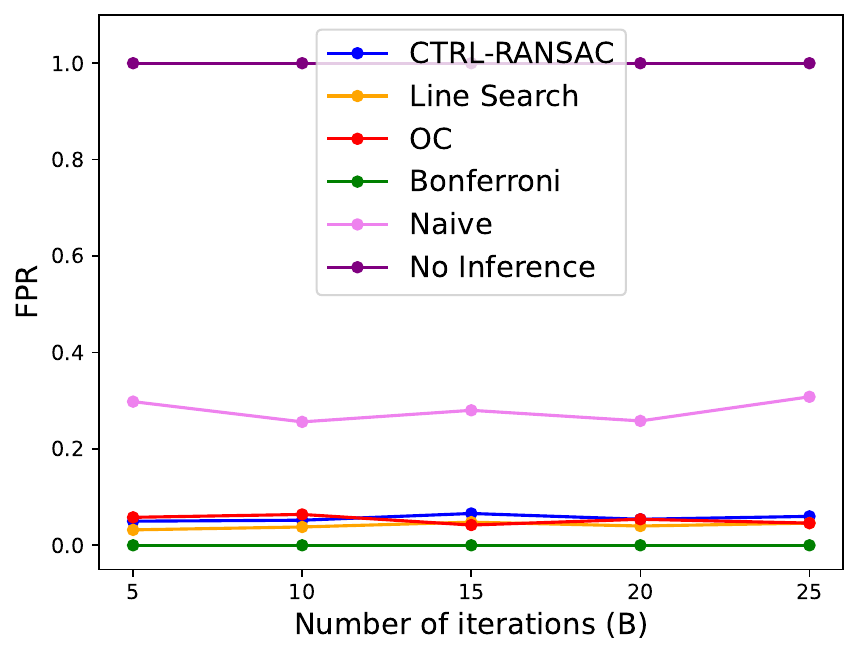}}\hfill
        \subfloat[]{\includegraphics[width=0.33\textwidth]{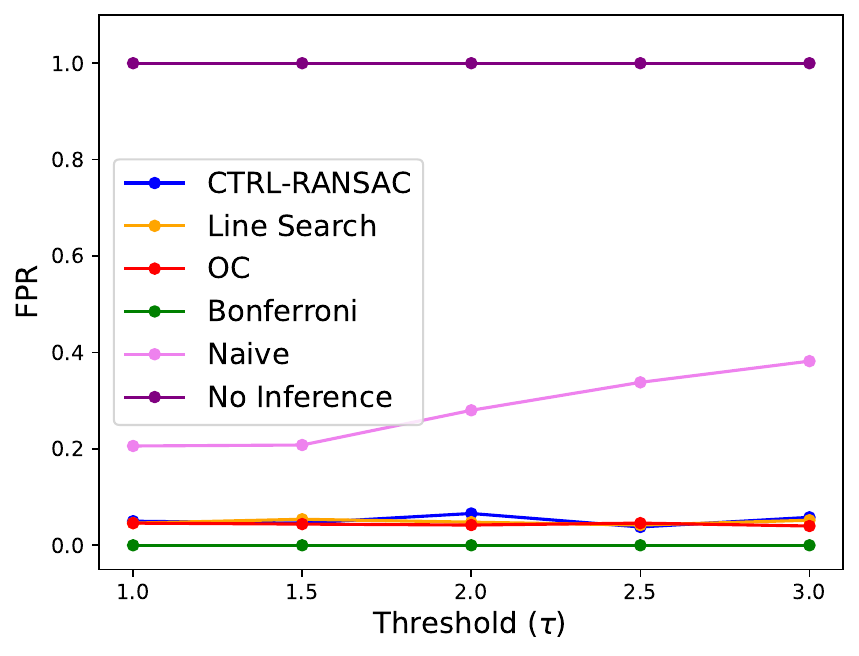}}\\
        (2) Correlation \\[0.5em]

        \caption{FPRs of \texttt{CTRL-RANSAC, Line Search, OC, Bonferroni, Naive} and \texttt{No inference} in various setups.}
        \label{fig:fpr_additional}
    \end{minipage}
\end{figure*}

\begin{figure*}[!t]
    \centering
    \begin{minipage}{\textwidth}
        \centering
        
        % First row with a row caption
        \subfloat[]{\includegraphics[width=0.24\textwidth]{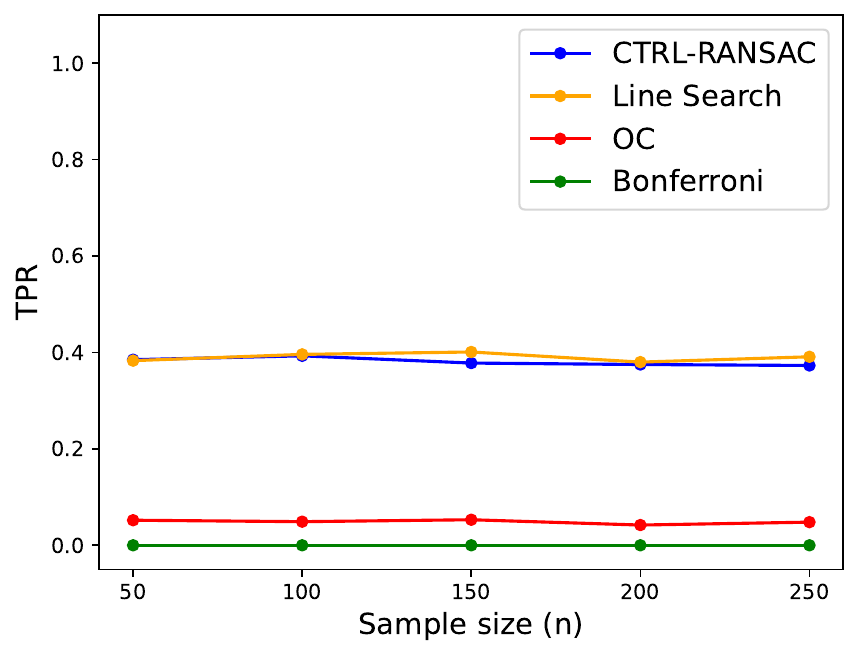}}\hfill
        \subfloat[]{\includegraphics[width=0.24\textwidth]{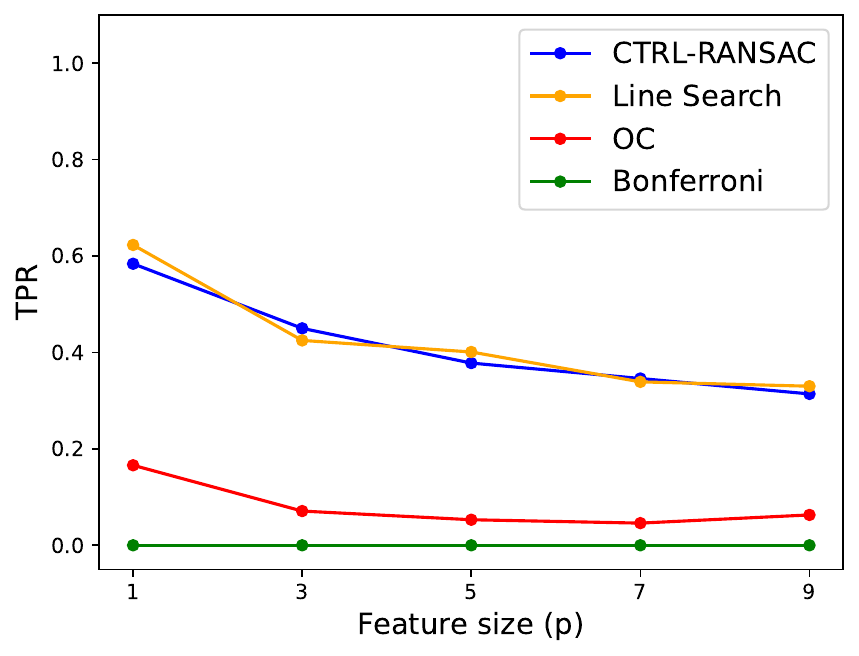}}\hfill
        \subfloat[]{\includegraphics[width=0.24\textwidth]{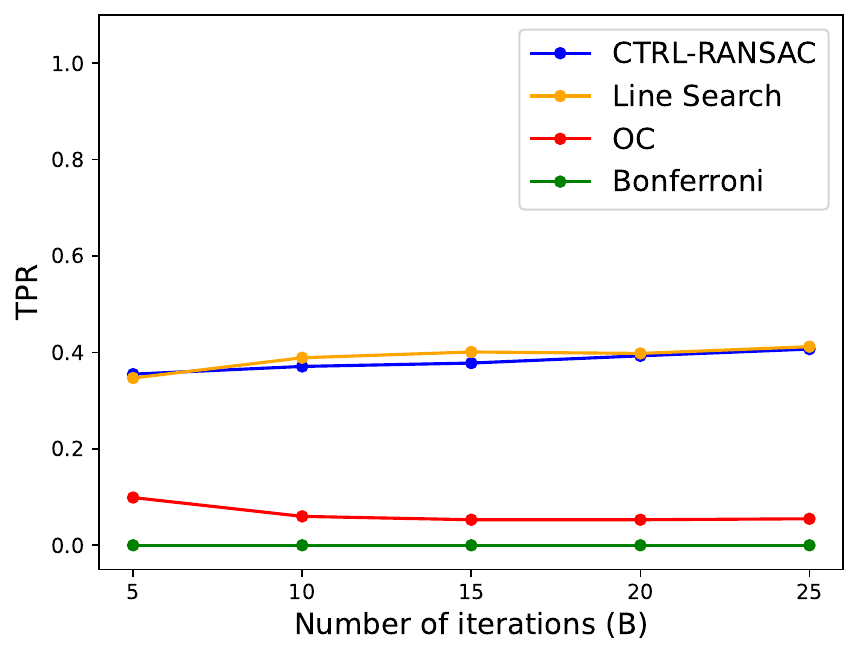}}\hfill
        \subfloat[]{\includegraphics[width=0.24\textwidth]{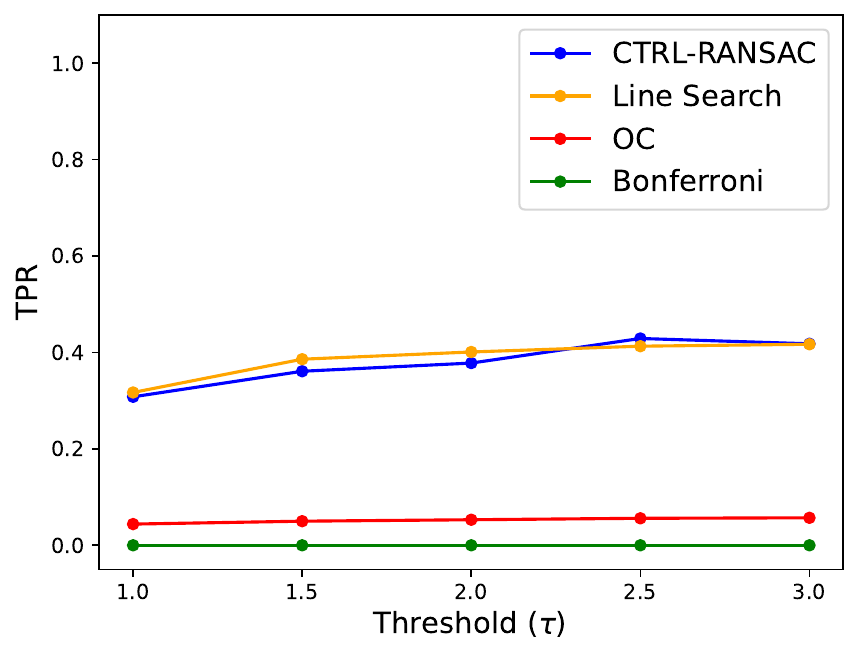}}\\
        (1) Independence \\[0.5em]
        
        % Second row with a row caption
        \subfloat[]{\includegraphics[width=0.24\textwidth]{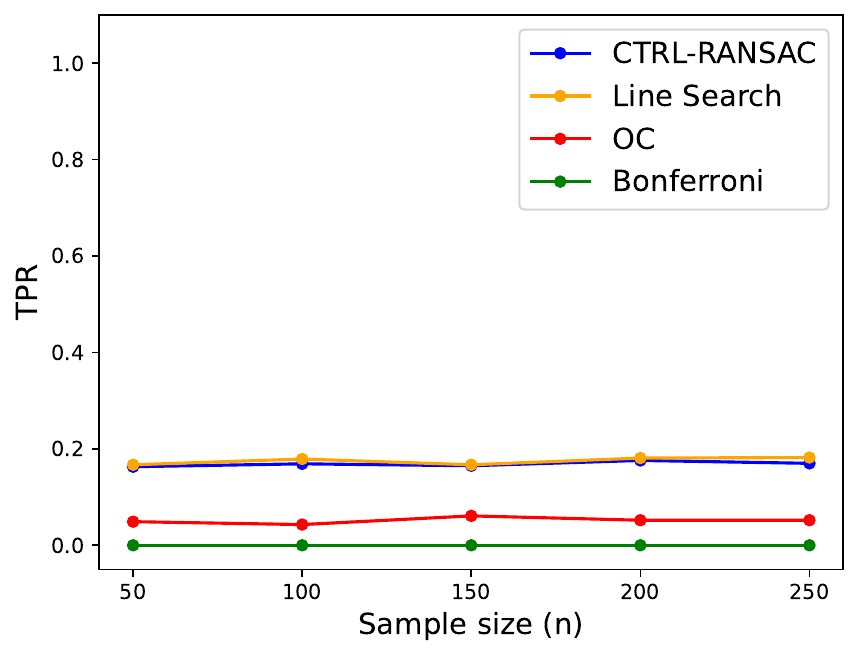}}\hfill
        \subfloat[]{\includegraphics[width=0.24\textwidth]{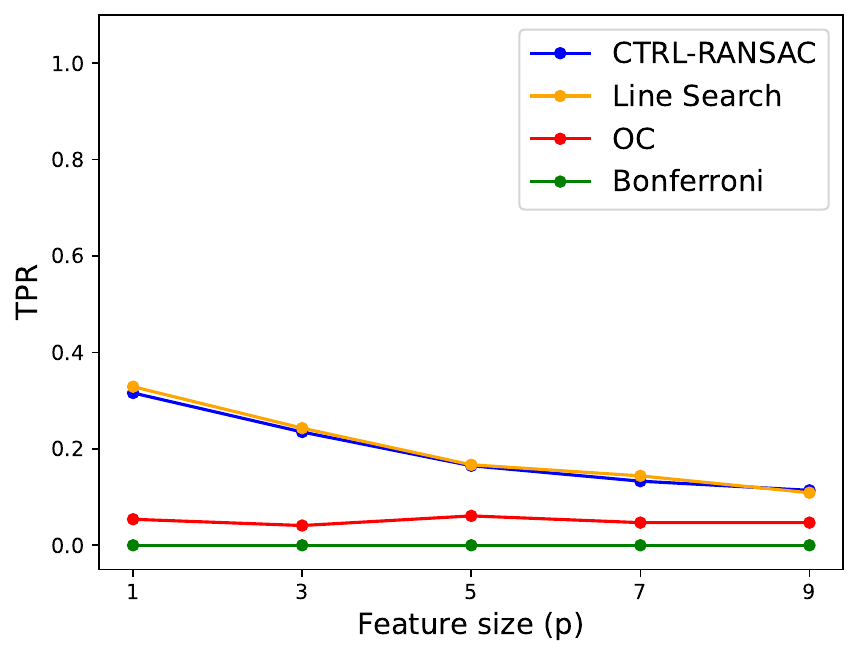}}\hfill
        \subfloat[]{\includegraphics[width=0.24\textwidth]{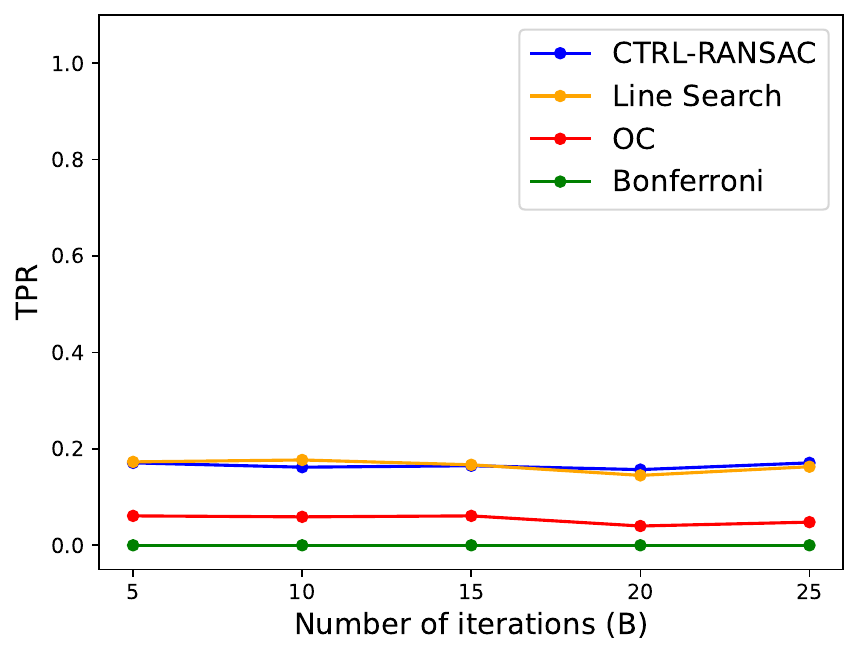}}\hfill
        \subfloat[]{\includegraphics[width=0.24\textwidth]{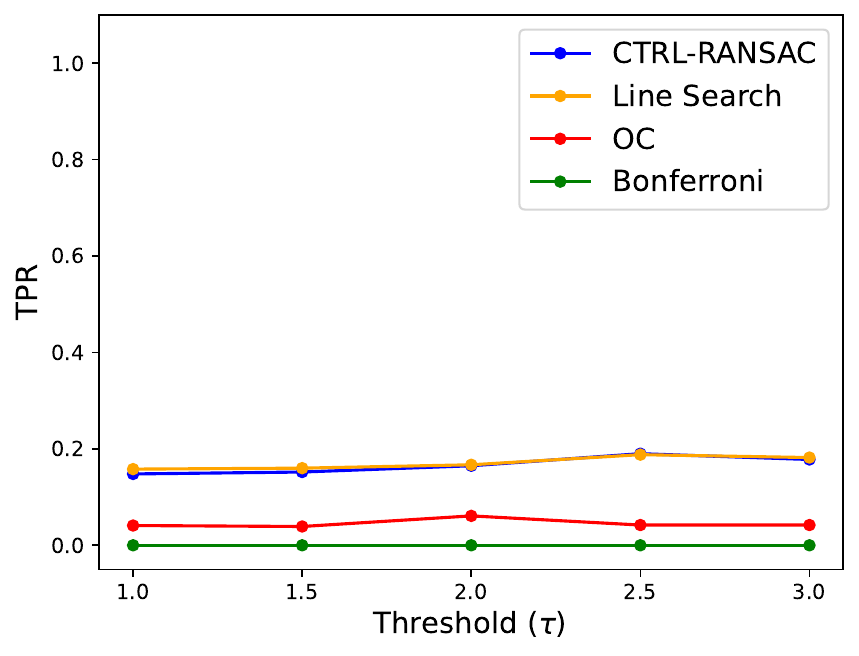}}\\
        (2) Correlation \\[0.5em]

        \caption{True positive rates of \texttt{CTRL-RANSAC, Line search, OC}, and \texttt{Bonferroni} in various setups.}
        \label{fig:tpr_additional}
    \end{minipage}
\end{figure*}

\textbf{Robustness experiments.} We conducted the following experiments:
\begin{itemize}
    \item Non-normal data: We considered the noise following Laplace distribution, skew normal distribution (skewness coefficient 10) and $t_{20}$ distribution. We set $n \in \{ 50, 100, 150, 200, 250 \}$, $p=5$, $B=15$ and $\tau = 2$. Each experiment was repeated 1000 times. We tested the FPR on both $\alpha = 0.05$ and $\alpha = 0.1$. The FPR results are shown in Fig. \ref{fig:robustness_p_value}a, \ref{fig:robustness_p_value}b, and \ref{fig:robustness_p_value}c. The CTRL-RANSAC controls the FPR well in the Skew normal distribution and the $t_{20}$ distribution. Meanwhile, it performs badly in Laplace distribution.
    \item Estimated variance: the variances of the noise were estimated from the data by using empirical variance. We confirm that CTRL-RANSAC still maintains a good performance on FPR control (Fig. \ref{fig:robustness_p_value}d).
\end{itemize}

\begin{figure}[!t]
    \centering
    \begin{subfigure}[t]{0.35\textwidth} \label{fig:robustness_laplace}
        \centering
        \includegraphics[width=\textwidth]{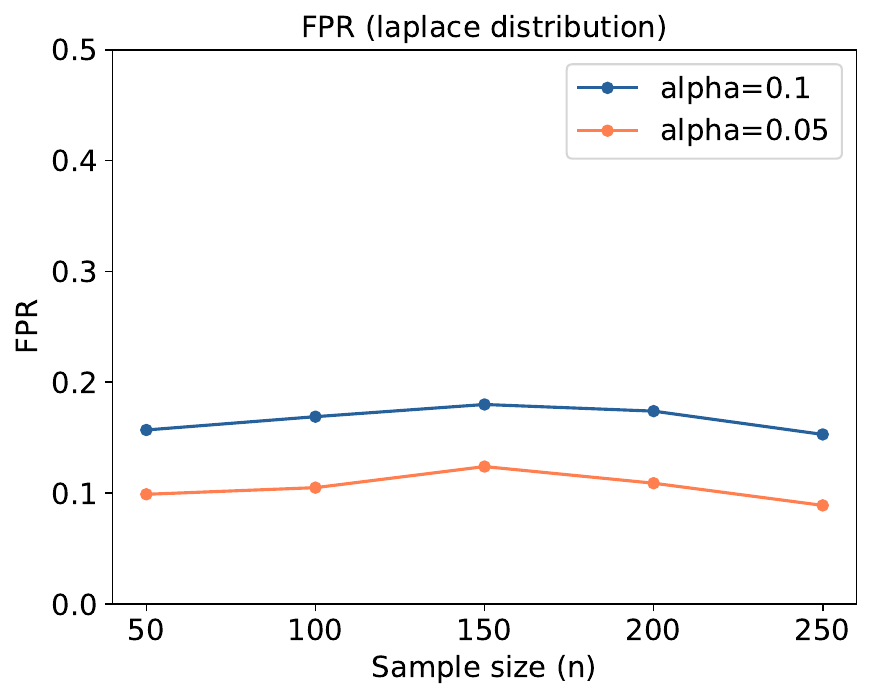}
        \caption{Laplace distribution}
    \end{subfigure}
    \hspace{2mm}
    \begin{subfigure}[t]{0.35\textwidth} \label{fig:robustness_skew}
        \centering
        \includegraphics[width=\textwidth]{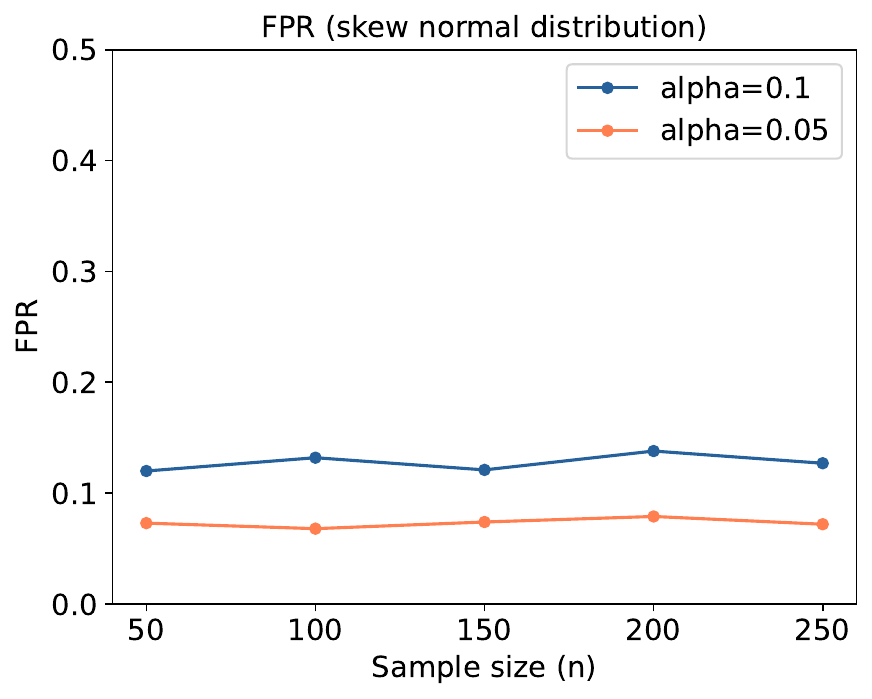}
        \caption{Skew normal distribution}
    \end{subfigure}

    \begin{subfigure}[t]{0.35\textwidth} \label{fig:robustness_t20}
        \centering
        \includegraphics[width=\textwidth]{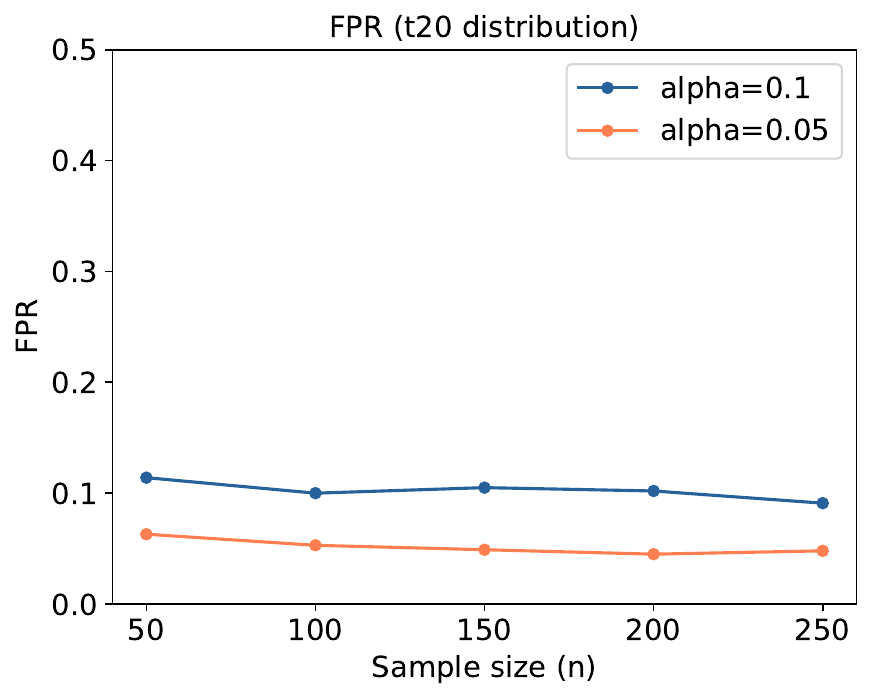}
        \caption{$t_{20}$ distribution}
    \end{subfigure}
    \hspace{2mm}
    \begin{subfigure}[t]{0.35\textwidth} \label{fig:robustness_estimate_sigma}
        \centering
        \includegraphics[width=\textwidth]{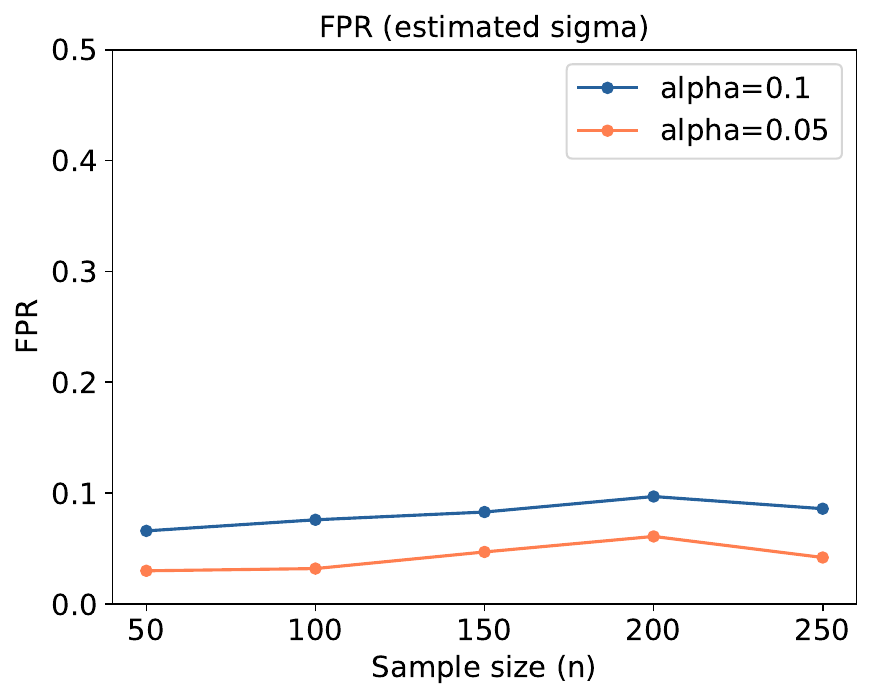}
        \caption{Estimated variance}
    \end{subfigure}
    
    \caption{False positive rate of the proposed CTRL-RANSAC method when data is non-normal or variance is unknown.}
    \label{fig:robustness_p_value}
\end{figure}

\subsection{Details of Line Search Method} \label{app:line_search_method}

\begin{algorithm}[!t]
\caption{\texttt{LineSearch}}
\label{alg:line_search}
\begin{footnotesize}
\textbf{Input:} $(X, \boldsymbol{Y}^{\rm obs}), z_{\rm min}, z_{\rm max}$
\begin{algorithmic}[1]
\vspace{2pt}
    \STATE $\mathcal{O}^{\rm obs} \gets$ Apply RANSAC to $(X, \boldsymbol{Y}^{\rm obs})$
    \vspace{2pt}
    \FOR{$i \in \mathcal{O}^{\rm obs}$}
    \vspace{2pt}
        \STATE Compute $\bm \eta_{i} \gets$ Eq. (\ref{eq:test_direction}), $\boldsymbol{a} \text{ and } \boldsymbol{b} \gets$ Eq. (\ref{eq:parametrized_response_vector})
        \vspace{2pt}
        \STATE $\cZ_{\rm line} \gets$ {\tt line\_search\_identify\_truncation\_region} ($\boldsymbol{a}, \boldsymbol{b}, \mathcal{O}^{\rm obs}, z_{\rm min}, z_{\rm max}$)
        \vspace{2pt}
        \STATE Compute $p^{\rm selective}_{i} \gets$ Eq. (\ref{eq:selective_p_reformulated}) with $\cZ = \cZ_{\rm line}$
        \vspace{2pt}
    \ENDFOR
\end{algorithmic}
\textbf{Output:} $\{p^{\rm selective}_{i}\}_{i \in \mathcal{O}^{\rm obs}}$
\end{footnotesize}
\end{algorithm}

\begin{algorithm}[!t]
\caption{\texttt{line\_search\_identify\_truncation\_region}}
\label{alg:line_search_identify_truncation_region}
\begin{footnotesize}
\textbf{Input:} $\bm a, \bm b, \cO^{\rm obs}, z_{\rm min}, z_{\rm max}$
\begin{algorithmic}[1]
\vspace{2pt}
    \STATE Initialization: $v = 1, z_v = z_{\rm min}, \cV = \emptyset$
    \WHILE{$z_v < z_{\rm max}$}
        \STATE $\cT_v \gets \text{ RANSAC-based AD on } \bm a + \bm b z_v$
        \STATE Compute $[L_v, R_v] = \cZ_v \gets$ Eq. \eq{eq:quadratic_inequalities_system_cZ_v}
        \IF{$\cM (\cT_v) = \cO^{\rm obs}$}
            \STATE $\cV = \cV \cup \{v\}$
        \ENDIF
        \STATE $v \gets v + 1, z_v = R_v$
    \ENDWHILE
    \STATE $\cZ_{\rm line} \gets \bigcup_{v \in \cV} \cZ_v$
\end{algorithmic}
\textbf{Output:} $\cZ_{\rm line}$
\end{footnotesize}
\end{algorithm}

Inspired by the works of \citet{le2021parametric} and \citet{le2024cad}, the Line Search method involves two key steps:
\begin{itemize}
    \item \textbf{Step 1 (extra-conditioning):} decomposing the problem into multiple sub-problems by conditioning on each selection event of RANSAC. Each sub-problem can then be solved through a finite number of operations.
    \item \textbf{Step 2 (line search):} combining multiple extra-conditioning steps and check the condition $\cO(\bm a + \bm b z) = \cO^{\rm obs}$ to obtain $\cZ_{\rm line}$.
\end{itemize}
Specifically, let $V$ be a number of all possible sets of steps performed by RANSAC along the parametrized line. The entire one-dimensional space $\RR$ can be decomposed as:
\begin{equation*}
    \RR = \bigcup_{v \in [V]} \left\{ z \in \RR 
    ~ \big| ~ 
    \cT_{\bm a + \bm b z} = \cT_v \right\},
\end{equation*}
where $\cT_{\bm a + \bm b z}$ denotes a set of steps performed by RANSAC when applying with $\bm Y (z) = \bm a + \bm b z$. Our goal is to search a set
\begin{equation*}
    \cV = \left\{ v : \cM (\cT_v) = \cO^{\rm obs} \right\},
\end{equation*}
for all $v \in [V]$, the function $\cM$ is defined as:
\begin{equation*}
    \cM : (\cT_{\bm a + \bm b z}) \mapsto \cO_{\bm a + \bm b z}.
\end{equation*}
Finally, the region $\cZ_{\rm line}$ can be obtained as follows:
\begin{align*}
    \cZ_{\rm line} 
    & = \left\{ z \in \RR ~ \big| ~ \cO_{\bm a + \bm b z} = \cO^{\rm obs} \right\} \\
    & =  \bigcup_{v \in \cV} \left\{ z \in  \RR ~ \big| ~ \cT_{\bm a + \bm b z} = \cT_v \right\}.
\end{align*}

\textbf{Extra-conditioning (step 1).} For any $v \in [V]$, we define the region for the extra-conditioning as:
\begin{equation*}
    \cZ_v = \left\{ z \in \RR ~ \big| ~ \cT_{\bm a + \bm b z} = \cT_v \right\}.
\end{equation*}
We derive $\cZ_v$ by solving a system of quadratic inequalities based on selection events of RANSAC. Specifically,
\begin{equation}
\label{eq:quadratic_inequalities_system_cZ_v}
    \cZ_v = \left\{ z \in \RR ~ \big| ~ \bm w + \bm r z + \bm o z^2 \leq \bm 0 \right\},
\end{equation}
where $\bm w$, $\bm r$, and $\bm o$ are the vectors specifically defined in the context of RANSAC.

% Theoretically, conditioning on each outlier set at every iteration and on the best model is necessary. However, intuitively, conditioning on every outlier set indirectly reveals the best model. Therefore, the extra-conditioning region is conditioned only on the outlier set at every iteration. We define it as:
% \begin{align*}
%     \cZ_c = \Big \{z\in \mathbb{R} \mid \forall u\in [B] : \cO^{(u)} (\bm a + \bm bz) = \cO^{(u)}_{\rm obs} \Big \},
% \end{align*}
% where, with a slight abuse of notation, $\cO^{(u)}_{\rm obs} = \cO^{(u)}(\bm Y^{\rm obs})$. In fact, the extra-conditioning region can be re-written in terms of quadratic inequalities w.r.t. $z$ as:
% \begin{align*}
%     \cZ_c = \bigcap_{u\in[B]}\Big(
%     \big(\cap_{i\notin\cO^{(u)}_{\rm obs}} \cR^{(u)}_i \big) 
%     \cap
%     \big(\cap_{i\in\cO^{(u)}_{\rm obs}} \overline{\cR}^{(u)}_i \big)
%     \Big),
% \end{align*}
% where $\cR^{(u)}_i$ is defined in (\ref{eq:cR_b_i}) and can be re-written in terms of quadratic inequalities w.r.t. $z$ as in \ref{app:proof_cZ_b_2}.

% When running the RANSAC algorithm, we have already known every set of outliers at every iteration. Then, for any $u\in[B]$, $i\notin \cO^{(u)}_{\rm obs}$, we determine the region of $z$ where the $i^{th}$ data point is NOT an outlier at the $u^{th}$ iteration of RANSAC. Similarly, for $i\in\cO^{(u)}_{\rm obs}$, we find the region of $z$ where the $i^{th}$ data point is an outlier. Finally, by intersecting all the identified regions, we obtain $\cZ_c$.

\textbf{Line search (step 2).} The strategy to identify $\cV$ by repeatedly applying RANSAC-based AD to a sequence of $\bm Y (z) = \bm a + \bm b z$ within sufficiently wide range of $z \in [z_{\rm min}, z_{\rm max}]$. For simplicity, we consider the case in which $\cZ_v$ is an interval, so we denote $\cZ_v = [L_v, R_v]$. The line search procedure is summarized in Algorithm \ref{alg:line_search_identify_truncation_region}. The entire steps of the Line Search method are summarized in Algorithm \ref{alg:line_search}.

\textbf{Complexity. } We analyze the complexity for each outlier in $\cO^{\rm obs}$. As shown in Algorithm \ref{alg:line_search}, the key task for each outlier is identifying the region $\cZ_{\rm line}$, which is the most computationally demanding step. This has a time complexity of $O(T\times B \times n)$, where $T$ is the number of iterations in the search procedure. Specifically, in each iteration of the search procedure, computing $\cZ_v$ requires $O(B \times n)$ operations, as it involves conditioning on the classification of each data point for each model. The main reason for using CTRL-RANSAC over the Line Search method lies in the number of iterations. Theoretically, $T$ can grow exponentially with respect to $n$ and $B$ in the worst case, making the Line Search method less computationally stable than CTRL-RANSAC.